%% file: main.tex
\documentclass{article}

\PassOptionsToPackage{numbers, compress}{natbib}
\usepackage[authoryear]{natbib}

\usepackage{fullpage}
\usepackage[utf8]{inputenc} \usepackage[T1]{fontenc}    \usepackage{hyperref}       \usepackage{url}            \usepackage{booktabs}       \usepackage{amsfonts}       \usepackage{nicefrac}       \usepackage{microtype}      
\title{Data-dependent Sample Complexity of Deep Neural Networks via Lipschitz Augmentation}

\author{%
	Colin Wei \thanks{Stanford University, email: \texttt{colinwei@stanford.edu}}~
	\and
	Tengyu Ma \thanks{Stanford University, email: \texttt{tengyuma@stanford.edu}}}

\input{macro_neurips}

\input{macro-specific}

\begin{document}

\maketitle

\newcommand{\cLip}{\mathcal{K}}
\newcommand{\depthlocal}{q}

\input{intro}

\input{related_work}

\input{notations_8pages}

\input{overview}

\input{graph_covering_8pages}

\input{lipschitz_8pages}

\input{graph_covering}

\input{lipschitz}

\input{neural_net_main}

\input{experiments}

\input{conclusion}

\appendix
\bibliography{refs}
\bibliographystyle{plainnat}
\newpage

\input{neural_net_application}

\input{app_computational_graph}

\input{app_augmentation}

\input{low_level_lipschitzfy}

\input{rnn}
\input{relu}

\input{experiment_details}
\input{helper_lemmas}

\end{document}

%% file: macro_neurips.tex
\usepackage[english]{babel}
\usepackage{amsmath}
\usepackage{amssymb}
\usepackage{amsfonts}
\usepackage{multirow}
\usepackage[export]{adjustbox}
\usepackage{amsmath,amssymb,amsthm}
\usepackage{hyperref}
\usepackage{bbm}
\usepackage{dsfont}
\usepackage{graphicx}

\usepackage{cases}
\usepackage[colorinlistoftodos]{todonotes}
\usepackage{wrapfig}

\def\shownotes{1}  \ifnum\shownotes=1
\newcommand{\authnote}[2]{$\ll$\textsf{\footnotesize #1 notes: #2}$\gg$}
\else
\newcommand{\authnote}[2]{}
\fi

\newtheorem{proposition}{Proposition}[section]
\newtheorem{theorem}{Theorem}[section]
\newtheorem{lemma}[theorem]{Lemma}
\newtheorem{definition}[theorem]{Definition}

\newtheorem{example}[theorem]{Example}

\newtheorem{claim}[theorem]{Claim}

\newcommand{\supp}{\text{supp}}

\newcommand{\E}{\mathbb{E}}

\newcommand{\R}{\mathbb{R}}
\newcommand{\1}{\mathbbm{1}}

\newcommand{\poly}{\textup{poly}}

\newcommand{\Gnorm}[1]{{\left\vert\kern-0.25ex\left\vert\kern-0.25ex\left\vert #1 
		\right\vert\kern-0.25ex\right\vert\kern-0.25ex\right\vert}}
	
	\newcommand{\gnorm}[1]{{\vert\kern-0.25ex\vert\kern-0.25ex\vert #1 
\vert\kern-0.25ex\vert\kern-0.25ex\vert}}

%% file: macro-specific.tex
\newcommand{\one}[1][]{\mathbbm{1}_{#1}}

\newcommand{\cover}{\mathcal{N}}
\newcommand{\opnorm}[1]{\|#1 \|_{\textup{op}}}
\newcommand{\Opnorm}[1]{\left\|#1\right\|_{\textup{op}}}
\newcommand{\ot}{\leftarrow}
\newcommand{\ulip}{\bar{\kappa}} \newcommand{\Exp}{\mathop{\mathbb E}\displaylimits}
\newcommand{\cut}[2]{#1^{\backslash #2}}
 \newcommand{\depthnn}{r} 

\newcommand{\figaugmentation}{3}

\newcommand{\augG}{\widetilde{\mathcal{G}}}
\newcommand{\augV}{\widetilde{\mathcal{V}}}
\newcommand{\augE}{\widetilde{\mathcal{E}}}
\newcommand{\augR}{\widetilde{\fR}}
\newcommand{\laug}{\tilde{\kappa}}

\newcommand{\depth}{p}

\newcommand{\pr}{\textup{pr}}
\newcommand{\rad}{\textup{Rad}_n} \newcommand{\kappaj}[1]{\kappa^{\textup{jacobian}, (#1)}}
\newcommand{\kapparj}[1]{\kappa^{\textup{relu-jacobian}, (#1)}}
\newcommand{\kappah}[1]{\kappa^{\textup{hidden}, (#1)}}
\newcommand{\kapparh}[1]{\kappa^{\textup{relu-hidden}, (#1)}}
\newcommand{\kapparnh}[1]{\kappa^{\textup{rnn-hidden}, (#1)}}
\newcommand{\kapparnj}[1]{\kappa^{\textup{rnn-jacobian}, (#1)}}

\newcommand{\cF}{\mathcal{F}}
\newcommand{\cG}{\mathcal{G}}
\newcommand{\cL}{\mathcal{L}}

\newcommand{\cU}{\mathcal{U}}
\newcommand{\cV}{\mathcal{V}}
\newcommand{\cZ}{\mathcal{Z}}

\newcommand{\fR}{\mathfrak{R}}
\newcommand{\cW}{\mathcal{W}}
\newcommand{\cI}{\mathcal{I}}

\newcommand{\cD}{\mathcal{D}}
\newcommand{\cE}{\mathcal{E}}
\newcommand{\cS}{\mathcal{S}}

\newcommand{\hath}{\hat{h}}

\newcommand{\hatcU}{\widehat{\mathcal{U}}}

\newcommand{\hatu}{\hat{u}}
\newcommand{\hatw}{\hat{w}}

\newcommand{\tildeG}{\widetilde{G}}

\newcommand{\tildeR}{\widetilde{R}}

%% file: intro.tex
\begin{abstract}
	Existing Rademacher complexity bounds for neural networks rely only on norm control of the weight matrices and depend exponentially on depth via a product of the matrix norms. Lower bounds show that this exponential dependence on depth is unavoidable when no additional properties of the training data are considered. We suspect that this conundrum comes from the fact that these bounds depend on the training data only through the margin. In practice, many \textit{data-dependent} techniques such as Batchnorm improve the generalization performance. For feedforward neural nets as well as RNNs, we obtain tighter Rademacher complexity bounds by considering additional data-dependent properties of the network: the norms of the hidden layers of the network, and the norms of the Jacobians of each layer with respect to all previous layers. Our bounds scale polynomially in depth when these empirical quantities are small, as is usually the case in practice. To obtain these bounds, we develop general tools for augmenting a sequence of functions to make their composition Lipschitz and then covering the augmented functions. Inspired by our theory, we directly regularize the network's Jacobians during training and empirically demonstrate that this improves test performance.
\end{abstract}
\section{Introduction}

Deep networks trained in practice typically use many more parameters than training examples, and therefore have the capacity to overfit to the training set~\citep{zhang2016understanding}. Fortunately, there are also many known (and unknown) sources of regularization during training: model capacity regularization such as simple weight decay, implicit or algorithmic regularization~\citep{gunasekar2017implicit,gunasekar2018implicit,soudry2018implicit,li2018algorithmic}, and finally regularization that depends on the training data such as Batchnorm~\citep{ioffe2015batch}, layer normalization~\citep{ba2016layer}, group normalization~\citep{wu2018group}, path normalization~\citep{neyshabur2015data}, dropout~\citep{srivastava2014dropout,wager2013dropout}, and regularizing the variance of activations~\citep{littwin2018regularizing}.

In many cases, it remains unclear why data-dependent regularization can improve the final test error ---  for example,  why Batchnorm empirically improves the generalization performance in practice~\citep{ioffe2015batch,zhang2018residual}. We do not have many tools for analyzing data-dependent regularization in the literature; with the exception of~\cite{dziugaite2018data},~\citep{arora2018stronger}~and~\citep{nagarajan2018deterministic} (with which we compare later in more detail), existing bounds typically consider properties of the weights of the learned model but little about their interactions with the training set. Formally, define a data-dependent property as any function of the learned model and the training data. 
In this work, we prove tighter generalization bounds by considering additional data-dependent properties of the network. Optimizing these bounds leads to data-dependent regularization techniques that empirically improve performance. 

One well-understood and important data-dependent property is the training margin:~\citet{bartlett2017spectrally} show that networks with larger normalized margins have better generalization guarantees. However, neural nets are  complex, so there remain many other data-dependent properties which could potentially lead to better generalization. We extend the bounds and techniques of~\citet{bartlett2017spectrally}~by considering additional properties: the hidden layer norms and interlayer Jacobian norms. 
Our final generalization bound (Theorem~\ref{thm:gen_union_bound}) is a polynomial in the hidden layer norms and Lipschitz constants on the training data. We give a simplified version below for expositional purposes. 
Let $F$ denote a neural network with \textit{smooth} activation $\phi$ parameterized by weight matrices $\{W^{(i)}\}_{i = 1}^r$ that perfectly classifies the training data with margin $\gamma > 0$. Let $t$ denote the maximum $\ell_2$ norm of any hidden layer or training datapoint, and $\sigma$ the maximum operator norm of any interlayer Jacobian, where both quantities are evaluated \textit{only on the training data}.  \begin{theorem}[Simplified version of Theorem~\ref{thm:gen_union_bound}]
	\label{thm:geninformal}
Suppose $\sigma, t \ge 1$. With probability $1- \delta$ over the training data, we can bound the test error of $F$ by \\
\resizebox{1\linewidth}{!}{
	\begin{minipage}{\linewidth}
\begin{align*}
L_{\textup{0-1}}(F) \le \widetilde{O}\left(\frac{  (\frac{\sigma }{\gamma} + r^3 \sigma^2) t \left(1 +\sum_i \|{W^{(i)}}^\top\|_{2,1}^{2/3}\right)^{3/2} + r^2\sigma\left(1+\sum_{i} \|{W^{(i)}}\|_{1, 1}^{2/3}\right)^{3/2}}{\sqrt{n}} + r\sqrt{ \frac{\log(\frac{1}{\delta})}{n}}\right) 
\end{align*}
\end{minipage}}

The notation $\tilde{O}$ hides logarithmic factors in $d,r,\sigma, t$ and the matrix norms. The $\|\cdot\|_{2,1}$ norm is formally defined in Section~\ref{sec:notations}.\end{theorem}
The degree of the dependencies on $\sigma$ may look unconventional --- this is mostly due to the dramatic simplification from our full Theorem~\ref{thm:gen_union_bound}, which obtains a more natural bound that considers all interlayer Jacobian norms instead of only the maximum. Our bound is polynomial in $t, \sigma$, and network depth, but independent of width. In practice, $t$ and $\sigma$ have been observed to be much smaller than the product of matrix norms~\citep{arora2018stronger,nagarajan2018deterministic}. We remark that our bound is not homogeneous because the smooth activations are not homogeneous and can cause a second order effect on the network outputs. 

In contrast, the bounds of~\citet{neyshabur2015norm,bartlett2017spectrally,neyshabur2017pac,golowich2017size}~all depend on a product of norms of weight matrices which scales exponentially in the network depth, and which can be thought of as a worst case Lipschitz constant of the network. In fact, lower bounds show that with only norm-based constraints on the hypothesis class, this product of norms is unavoidable for Rademacher complexity-based approaches (see for example Theorem 3.4 of~\citep{bartlett2017spectrally}~and Theorem 7 of~\citep{golowich2017size}). We circumvent these lower bounds by additionally considering the model's Jacobian norms -- empirical Lipschitz constants which are much smaller than the product of norms because they are only computed on the training data. 

The bound of~\citet{arora2018stronger}~depends on similar quantities related to noise stability but only holds for a compressed network and not the original. 
The bound of~\citet{nagarajan2018deterministic}~also depends polynomially on the Jacobian norms rather than exponentially in depth; however these bounds also require that the inputs to the activation layers are bounded away from 0, an assumption that does not hold in practice~\citep{nagarajan2018deterministic}. We do not require this assumption because we consider networks with smooth activations, whereas the bound of~\citet{nagarajan2018deterministic} applies to relu nets.

In Section~\ref{sec:recurrent}, we additionally present a generalization bound for recurrent neural nets that scales polynomially in the same quantities as our bound for standard neural nets. Prior generalization bounds for RNNs either require parameter counting~\citep{koiran1997vapnik}~or depend exponentially on depth~\citep{zhang2018stabilizing,chen2019on}. 
\begin{figure}
	\begin{minipage}[c]{0.4\textwidth}
		\includegraphics[width=\textwidth]{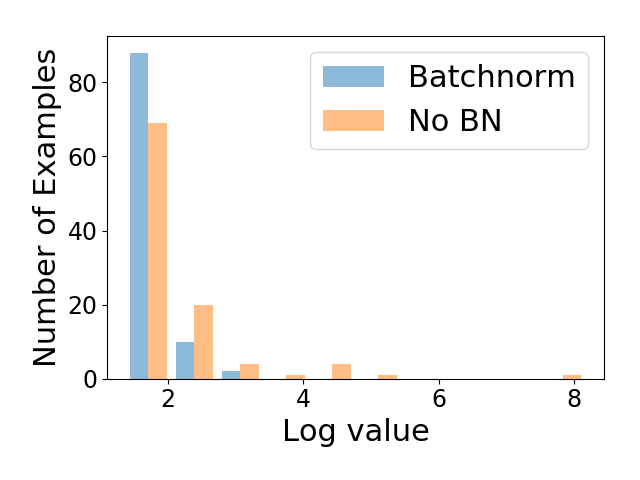}
	\end{minipage}\hfill
	\begin{minipage}[c]{0.5\textwidth}
			\caption{Let $h_1, h_2, h_3$ denote the 1st, 2nd, and 3rd blocks of a 16-layer WideResNet and $J_i$ the Jacobian of the output w.r.t layer $i$. In log-scale we plot a histogram of the 100 largest values on the training set of $\sum_{i = 1}^3 \|h_i\| \|J_i\|/\gamma$ for a WideResNet trained with and without Batchnorm on CIFAR10, where $\gamma$ is the example's margin.}	
		\label{fig:figbatchnorm}	
	\end{minipage}
	\vspace{-0.5cm}
\end{figure}

In Figure~\ref{fig:figbatchnorm}, we plot the distribution over the sum of products of Jacobian and hidden layer norms (which is the leading term of the bound in our full Theorem~\ref{thm:gen_union_bound}) for a WideResNet~\citep{zagoruyko2016wide} trained with and without Batchnorm. Figure~\ref{fig:figbatchnorm}~shows that this sum blows up for networks trained without Batchnorm, indicating that the terms in our bound are empirically relevant for explaining data-dependent regularization.

An immediate bottleneck in proving Theorem~\ref{thm:geninformal} is that standard tools require fixing the hypothesis class before looking at training data, whereas conditioning on data-dependent properties makes the hypothesis class a random object depending on the data. A natural attempt is to augment the loss with indicators on the intended data-dependent quantities $\{\gamma_i\}$, with desired bounds $\{\kappa_i\}$ as follows:
\begin{align*}
l_{\textup{aug}} = (l_{\textup{old}}-1) \prod_{\textup{properties }\gamma_i}\one(\gamma_i \le \kappa_i) +1 \end{align*}
This augmented loss upper bounds the original loss $l_{\textup{old}} \in [0, 1]$, with equality when all properties hold for the training data. The augmentation lets us reason about a hypothesis class that is independent of the data by directly conditioning on data-dependent properties in the loss. The main challenges with this approach are twofold: 1) designing the correct set of properties and 2) proving generalization of the final loss $l_{\textup{aug}}$, a complicated function of the network. 

Our main tool is covering numbers: Lemma~\ref{lem:composition-simple-k} shows that a composition of functions (i.e, a neural network) has low covering number if the output is worst-case Lipschitz at each level of the composition and internal layers are bounded in norm. Unfortunately, the standard neural net loss satisfies neither of these properties (without exponential dependencies on depth). However, by augmenting with properties $\gamma$, we can guarantee they hold. One technical challenge is that augmenting the loss makes it harder to reason about covering, as the indicators can introduce complicated dependencies between layers. 

Our main technical contributions are: 1) We demonstrate how to augment a composition of functions to make it Lipschitz at all layers, and thus easy to cover. Before this augmentation, the Lipschitz constant could scale exponentially in depth (Theorem~\ref{thm:inf_lip}).  2) We reduce covering a complicated sequence of operations to covering the individual operations (Theorem~\ref{thm:inf_covering_graph}).  3) By combining 1 and 2, it follows cleanly that our augmented loss on neural networks has low covering number and therefore has good generalization. Our bound scales polynomially, not exponentially, in the depth of the network when the network has good Lipschitz constants on the training data (Theorem~\ref{thm:gen_union_bound}). 

As a complement to the main theoretical results in this paper, we show empirically in Section~\ref{sec:experiments} that directly regularizing our complexity measure can result in improved test performance.

%% file: related_work.tex
\section{Related Work}\label{sec:related}

\citet{zhang2016understanding}~and~\citet{neyshabur2017exploring}~show that generalizaton in deep learning often disobeys conventional statistical wisdom. One of the approaches adopted torwards explaining generalization is implicit regularization; numerous recent works have shown that the training method prefers minimum norm or maximum margin solutions~\citep{soudry2018implicit,li2018algorithmic,ji2018risk,gunasekar2017implicit,gunasekar2018characterizing,gunasekar2018implicit,wei2018margin}. With the exception of~\citep{wei2018margin}, these papers analyze simplified settings and do not apply to larger neural networks.  

This paper more closely follows a line of work related to Rademacher complexity bounds for neural networks \citep{neyshabur2015norm,neyshabur2018towards,bartlett2017spectrally,golowich2017size,li2018tighter}. For a comparison, see the introduction. There has also been work on deriving PAC-Bayesian bounds for generalization~\citep{neyshabur2017exploring,neyshabur2017pac,nagarajan2018deterministic}.~\citet{dziugaite2017computing} optimize a bound to compute non-vacuous bounds for generalization error. Another line of work analyzes neural nets via their behavior on noisy inputs.~\citet{neyshabur2017exploring}~prove PAC-Bayesian generalization bounds for random networks under assumptions on the network's empirical noise stability.~\citet{arora2018stronger}~develop a notion of noise stability that allows for compression of a network under an appropriate noise distribution. They additionally prove that the compressed network generalizes well. In comparison, our Lipschitzness construction also relates to noise stability, but our bounds hold for the original network and do not rely on the particular noise distribution. 

\citet{nagarajan2018deterministic}~use PAC-Bayes bounds to prove a similar result as ours for generalization of a network with bounded hidden layer and Jacobian norms. The main difference is that their bounds depend on the inverse relu preactivations, which are found to be large in practice~\citep{nagarajan2018deterministic}; our bounds apply to smooth activations and avoid this dependence at the cost of an additional factor in the Jacobian norm (shown to be empirically small). We note that the choice of smooth activations is empirically justified~\citep{clevert2015fast,klambauer2017self}. We also work with Rademacher complexity and covering numbers instead of the PAC-Bayes framework. It is relatively simple to adapt our techniques to relu networks to produce a similar result to that of \citet{nagarajan2018deterministic}, by conditioning on large pre-activation values in our Lipschitz augmentation step (see Section~\ref{subsec:lipschitzaug_8}). In Section~\ref{app:relu}, we provide a sketch of this argument and obtain a bound for relu networks that is polynomial in hidden layer and Jacobian norms and inverse preactivations. However, it is not obvious how to adapt the argument of~\citet{nagarajan2018deterministic}~to activation functions whose derivatives are not piecewise-constant.

\citet{dziugaite2018data,dziugaite2017entropy}~develop PAC-Bayes bounds for data-dependent priors obtained via some differentially private mechanism. Their bounds are for a randomized classifier sampled from the prior, whereas we analyze a deterministic, fixed model.

\citet{novak2018sensitivity}~empirically demonstrate that the sensitivity of a neural net to input noise correlates with generalization.~\citet{sokolic2017robust,krueger2015regularizing} propose stability-based regularizers for neural nets.~\citet{hardt2015train} show that models which train faster tend to generalize better.~\citet{keskar2016large,hoffer2017train}~study the effect of batch size on generalization.~\citet{brutzkus2017sgd}~analyze a neural network trained on hinge loss and linearly separable data and show that gradient descent recovers the exact separating hyperplane.

%% file: notations_8pages.tex
\section{Notation}\label{sec:notations}
Let $\one(E)$ be the indicator function of event $E$. Let $l_{\textup{0-1}}$ denote the standard 0-1 loss. For $\kappa\ge 0$, Let $\one[\le \kappa](\cdot)$ be the softened indicator function defined as 
\begin{align}
\one[\le \kappa](t) =  \left\{\begin{array}{ll}
1 & \textup{ if } t \le \kappa \\
2 - t/\kappa & \textup{ if } \kappa \le t \le 2\kappa \\
0 & \textup{ if } 2\kappa \le t  \\
\end{array}\right.\nonumber
\end{align}
Note that $\one[\le \kappa]$ is $\kappa^{-1}$-Lipschitz. Define the norm $\| \cdot \|_{p, q}$ by $\|A\|_{p, q} \triangleq \Big(\sum_{j} \big(\sum_i A_{i, j}^p \big)^{q/p}\Big)^{1/q}$.
Let $P_n$ be a uniform distribution over $n$ points $\{x_1,\dots, x_n\}\subset \cD_x$. Let $f$ be a function that maps $\cD_x$ to some output space $\cD_f$, and assume both spaces are equipped with some norms $\gnorm{\cdot}$ (these norms can be different but we use the same notations for them). Then the $L_2(P_n, \gnorm{\cdot})$ norm of the function $f$ is defined as 
$$
\|f\|_{L_2(P_n, \gnorm{\cdot})} \triangleq  \left(\frac{1}{n} \sum_i \gnorm{f(x_i)}^2\right)^{1/2}
$$
We use $D$ to denote total derivative operator, and thus $Df(x)$ represents the Jacobian of $f$ at $x$. Suppose $\cF$ is a family of functions from $\cD_x$ to $\cD_f$. 
Let $\mathcal{C}(\epsilon, \mathcal{F}, \rho)$ be the covering number of the function class $\cF$ w.r.t. metric $\rho$ with cover size $\epsilon$. 
In many cases, the covering number depends on the examples through the norms of the examples, and in this paper we only work with these cases. 
Thus, we let $\cover(\epsilon,\mathcal{F}, s)$ be the maximum covering number for any possible $n$ data points with norm not larger than $s$. 
Precisely, if we define $\mathcal{P}_{n,s}$ to be the set of all possible uniform distributions supported on $n$ data points with norms not larger than $s$, then 
$$
\cover(\epsilon, \mathcal{F}, s) \triangleq \sup_{P_n \in \mathcal{P}_{n,s}} \mathcal{C}(\epsilon, \mathcal{F}, L_2(P_n, \gnorm{\cdot}))
$$
Suppose $\mathcal{F}$ contains functions with $m$ inputs that map from a tensor product $m$ Euclidean space to Euclidean space, then we define 
$$
\cover(\epsilon, \mathcal{F}, (s_1,\dots, s_m)) \triangleq \sup_{\substack{P: \forall (x_1,\dots, x_m) \in \supp(P)\\ \|x_i\|\le s_i}} \mathcal{C}(\epsilon, \mathcal{F}, L_2(P))
$$

%% file: overview.tex
\section{Overview of Main Results and Proof Techniques}
\label{sec:overview}
In this section, we give a general overview of the main technical results and outline how to prove them with minimal notation. We will point to later sections where many statements are formalized. 

To simplify the core mathematical reasoning, we abstract feed-forward neural networks (including residual networks) as compositions of operations. Let $\cF_1,\dots, \cF_k$ be a sequence of families of functions (corresponding to families of single layer neural nets in the deep learning setting) and $\ell$ be a Lipschitz loss function taking values in $[0,1]$. We study the compositions of $\ell$ and functions in $\cF_i$'s: 
\begin{align}
\mathcal{L} \triangleq \ell\circ \cF_k \circ \cF_{k-1}\cdots \circ \cF_1 = \{\ell \circ f_k \circ f_{k-1}\circ \cdots \circ f_1 : \forall i, f_i \in \cF_i\} \label{eq:sequential_family}
\end{align}
Textbook results~\citep{bartlett2002rademacher} bound the generalization error by the Rademacher complexity (formally defined in Section~\ref{sec:neural_net_app}) of the family of losses $\mathcal{L}$, which in turn is bounded by the covering number of $\cL$ through Dudley's entropy integral theorem~\citep{dudley1967sizes}. 
Modulo minor nuances, the key remaining question is to give a tight covering number bound for the family $\cL$ for every target cover size $\epsilon$ in a certain range (often, considering $\epsilon \in [1/n^{O(1)}, 1]$ suffices). 

As alluded to in the introduction, generalization error bounds obtained through this machinery only depend on the (training) data through the margin in the loss function, and our aim is to utilize more data-dependent properties. 
Towards understanding which data-dependent properties are useful to regularize, it is helpful to revisit the data-independent covering technique of~\citep{bartlett2017spectrally}, the skeleton of which is summarized below. 

Recall that $\cover(\epsilon, \cF, s)$ denotes the covering number for arbitrary $n$ data points with norm less than $s$. 
The following lemma says that if the intermediate variable (or the hidden layer) $f_{i}\circ \cdots \circ f_1(x)$ is bounded, and the composition of the rest of the functions $l \circ f_{k}\circ \cdots \circ f_{i + 1}(x)$ is Lipschitz, then small covering number of local functions imply small covering number for the composition of functions. 
\begin{lemma}\label{lem:composition-simple-k}[abstraction of techniques in~\citep{bartlett2017spectrally}]
In the context above, assume: 
	\begin{itemize}
		\item[1.] for any $x\in \supp(P_n)$, $\gnorm{f_{i}\circ \cdots \circ f_1(x)} \le s_i$. 
		\item[2.] $\ell\circ f_k \circ \cdots \circ f_{i+1}$ is $\kappa_{i}$-Lipschitz for all $i$.  
	\end{itemize}
	Then, we have the following covering number bound for $\cL$ (for any choice of $\epsilon_1,\dots, \epsilon_k > 0$):	$$
	\log \cover(\sum_{i=1}^{k}\kappa_i \epsilon_i, \cL, s_0) \le \sum_{i=1}^k 	\log \cover(\epsilon_i, \cF_i, s_{i - 1})
	$$
\end{lemma}
The lemma says that the log covering number and the cover size scale linearly if the Lipschitzness parameters and norms remain constant. However, these two quantities, in the worst case, can easily scale exponentially in the number of layers, and they are the main sources of the dependency of product of spectral/Frobenius norms of layers in~\citep{golowich2017size,bartlett2017spectrally,neyshabur2017pac,neyshabur2015norm} More precisely, the worst-case Lipschitzness over all possible data points can be exponentially bigger than the average/typical Lipschitzness for examples randomly drawn from the training or test distribution. We aim to bridge this gap by deriving a generalization error bound that only depends on the Lipschitzness and boundedness on the training examples. 

Our general approach, partially inspired by margin theory, is to augment the loss function by soft indicators of Lipschitzness and boundedness. Let $h_i$ be shorthand notation for $f_i\circ \cdots \circ f_1$, the $i$-th intermediate value, and let $z(x) \triangleq \ell(h_k(x))$ be the original loss. Our first attempt considered:
\begin{align}
\tilde{z}'(x)  \triangleq 1 +  (z(x)-1) \cdot \prod_{i = 1}^k \one[\le s_i] (\|h_i(x)\|) \cdot \prod_{i=1}^k \one[\le \kappa_{i}] (\Opnorm{{\partial z}/{\partial h_i}}) \label{eq:z_prime_expr}
\end{align}
Since $z$ takes values in $[0, 1]$, the augmented loss $\tilde{z}'$ is an upper bound on the original loss $z$ with equality when all the indicators are satisfied with value $1$. The hope was that the indicators would flatten those regions where $h_i$ is not bounded and where $z$ is not Lipschitz in $h_i$. 
However, there are two immediate issues. 
 First, the soft indicators functions are themselves functions of $h_i$. It's unclear whether the augmented function can be Lipschitz with a small constant w.r.t $h_i$, and thus we cannot apply Lemma~\ref{lem:composition-simple-k}.\footnote{A priori, it's also unclear what ``Lipschitz in $h_i$'' means since the $\bar{z}'$ does not only depend on $x$ through $h_i$. We will formalize this in later section after defining proper language about dependencies between variables.\label{footnote:1}} 
Second, the augmented loss function becomes complicated and doesn't fall into the sequential computation form of Lemma~\ref{lem:composition-simple-k}, and therefore even if Lipschitzness is not an issue, we need new covering techniques beyond Lemma~\ref{lem:composition-simple-k}. 

We address the first issue by \textit{recursively} augmenting the loss function by multiplying more soft indicators that bound the Jacobian of the current function. The final loss $\tilde{z}$ reads:\footnote{Unlike in equation~\eqref{eq:z_prime_expr}, we don't augment the Jacobian of the loss w.r.t the layers. This allows us to deal with non-differentiable loss functions such as ramp loss.}
\begin{align}
\tilde{z}(x)  \triangleq 1 +  (z(x)-1) \cdot \prod_{i = 1}^k \one[\le s_i] (\|h_i(x)\| ) \cdot \prod_{1\le i\le j\le k} \one[\le \kappa_{j\ot i}] (\opnorm{D f_j\circ\cdots\circ f_{i}[h_{i-1}] }) \label{eq:tilde_z_loss}
\end{align}
where $\kappa_{j \ot i}$'s are user-defined parameters. For our application to neural nets, we instantiate $s_i$ as the maximum norm of layer $i$ and $\kappa_{j \ot i}$ as the maximum norm of the Jacobian between layer $j$ and $i$ across the training dataset. A polynomial in $\kappa, s$ can be shown to bound the worst-case Lipschitzness of the function w.r.t. the intermediate variables in the formula above.\footnote{As mentioned in footnote~\ref{footnote:1}, we will formalize the precise meaning of Lipschitzness later.} By our choice of $\kappa$, $s$, a) the training loss is unaffected by the augmentation and b) the worst-case Lipschitzness of the loss is controlled by a polynomial of the Lipschitzness on the training examples. We provide an informal overview of our augmentation procedure in Section~\ref{subsec:lipschitzaug_8} and formally state definitions and guarantees in Section~\ref{sec:lipschitzaug}. 
The downside of the Lipschitz augmentation is that it further complicates the loss function. Towards covering the loss function (assuming Lipschitz properties) efficiently, we extend Lemma~\ref{lem:composition-simple-k}, which works for sequential compositions of functions, to general families of formulas, or computational graphs. We informally overview this extension in Section~\ref{sec:computation_graph_8} using a minimal set of notations, and in Section~\ref{sec:computation_graph}, we give a formal presentation of these results.

Combining the Lipschitz augmentation and graphs covering results, we obtain a covering number bound of augmented loss. The theorem below is formally stated in Theorem~\ref{thm:sequential_covering}~of Section~\ref{sec:lipschitzaug}. 
\begin{theorem} \label{thm:sequential_cover_overview}
	Let $\tilde{\cL}$ be the family of augmented losses defined in~\eqref{eq:tilde_z_loss}. For cover resolutions $\epsilon_i$ and values $\tilde{\kappa}_i$ that are polynomial in the parameters $s_i, \kappa_{j \ot i}$, we obtain the following covering number bound for $\tilde{\cL}$:
	\begin{align*}
	\log \cover(\sum_i \epsilon_i \tilde{\kappa}_i, \tilde{\cL}, s_0) \le \sum_i \log \cover(\epsilon_i, \cF_i, s_{i - 1}) + \sum_i \log \cover(\epsilon_i, D\cF_i, s_{i - 1})
	\end{align*}
	where $D\cF_i$ denotes the function class obtained from applying the total derivative operator to all functions in $\cF_i$.
\end{theorem}
\sloppy Now, following the standard technique of bounding Rademacher complexity via covering numbers, we can obtain generalization error bounds for augmented loss.  
	For the demonstration of our technique, suppose that the following simplification holds:
$$
	\log \cover(\epsilon_i, D\cF_i, s_{i - 1}) = \log \cover(\epsilon_i, \cF_i, s_{i - 1}) = s_{i- 1}^2/\epsilon_i^2
$$ 
	Then after minimizing the covering number bound in $\epsilon_i$ via standard techniques, we obtain the below generalization error bound on the original loss for parameters $\tilde{\kappa}_i$ alluded to in Theorem~\ref{thm:sequential_cover_overview} and formally defined in Theorem~\ref{thm:lip}. When the training examples satisfy the augmented indicators, $\Exp_{\textup{train}}[\tilde{z}] = \Exp_{\textup{train}}[z]$, and because $\tilde{z}$ bounds $z$ from above, we have 
	\begin{align}
	\label{eq:generalization_sketch}
	\Exp_{\textup{test}}\left[z\right] - \Exp_{\textup{train}}\left[z\right] \le \Exp_{\textup{test}}\left[\tilde{z}\right] - \Exp_{\textup{train}}\left[\tilde{z}\right] \le  \widetilde{O}\Bigg(\frac{\left(\sum_{i} \tilde{\kappa}_i^{2/3} s_{i - 1}^{2/3}\right)^{3/2}}{\sqrt{n}}+ \sqrt{\frac{\log(1/\delta)}{n}}\Bigg)
	\end{align}

%% file: graph_covering_8pages.tex
\subsection{Overview of Computational Graph Covering} \label{sec:computational_graph_8}
\label{sec:computation_graph_8}  

To obtain the augmented $\tilde{z}$ defined in~\eqref{eq:tilde_z_loss}, we needed to condition on data-dependent properties which introduced dependencies between the various layers. Because of this, Lemma~\ref{lem:composition-simple-k}~is no longer sufficient to cover $\tilde{z}$. In this section, we informally overview how to extend Lemma~\ref{lem:composition-simple-k}~to cover more general functions via the notion of computational graphs. Section~\ref{sec:computational_graph} provides a more formal discussion. 

A computational graph $G(\cV,\cE,\{R_V\})$ is an acyclic directed graph with three components: the set of nodes $\cV$ corresponds to variables, the set of edges $\cE$ describes dependencies between these variables, and $\{R_V\}$ contains a list of composition rules indexed by the variables $V$'s, representing the process of computing $V$ from its direct predecessors. 
For simplicity, we assume the graph contains a unique sink, denoted by $O_G$, and we call it the ``output node''. 
We also overload the notation $O_G$ to denote the function that the computational graph $G$ finally computes.
Let $\cI_G= \{I_1,\dots, I_{p}\}$ be the subset of nodes with no predecessors, which we call the ``input nodes'' of the graph. 

The notion of a family of computational graphs generalizes the sequential family of function compositions in~\eqref{eq:sequential_family}.  Let $\cG = \{G(\cV, \cE, \{R_V\})\}$ be a family of computational graphs with shared nodes, edges, output node, and input nodes (denoted by $\cI$). Let $\fR_V$ be the collection of all possible composition rules used for node $V$ by the graphs in the family $\cG$. 
This family $\cG$ defines a set of functions $O_\cG \triangleq \{O_G: G\in \cG\}$. 

The theorem below extends Lemma~\ref{lem:composition-simple-k}. In the computational graph interpretation, Lemma~\ref{lem:composition-simple-k} applies to a sequential family of computational graphs with $k$ internal nodes $V_1, \ldots, V_k$, where each $V_i$ computes the function $f_i$, and the output computes the composition $O_G = \ell \circ f_k \cdots \circ f_1 = z$. However, the augmented loss $\tilde{z}$ no longer has this sequential structure, requiring the below theorem for covering generic families of computational graphs. We show that covering a general family of computational graphs can be reduced to covering all the local composition rules.

 \begin{theorem}[Informal and weaker version of Theorem~\ref{thm:covering_graph}]
	\label{thm:inf_covering_graph}
	Suppose that there is an ordering $(V_1,\dots, V_m)$ of the nodes, so that after cutting out nodes $V_1,\dots, V_{i-1}$, the node $V_i$ becomes a leaf node and the output $O_G$ is $\kappa_{V_i}$-Lipschitz w.r.t to $V_i$ for all $G\in \cG$. In addition, assume that for all $G\in \cG$, the node $V$'s value has norm at most $s_V$. Let $\pr(V)$ be all the predecessors of $V$ and $s_{\pr(V)}$ be the list of norm upper bounds of the predecessors of $V$.

	Then, small covering numbers for all of the local composition rules of $V$ with resolution $\epsilon_V$ would imply small covering number for the family of computational graphs with resolution $\sum_V \epsilon_V\kappa_V$:
		\begin{align}
		\label{eq:covering_graph-1}
		\log \cover(\sum_{V\in \cV \setminus \cI \cup \{O\}} \kappa_V \epsilon_V + \epsilon_O, O_\cG, s_{\cI}) \le   \sum_{V\in \cV \setminus \cI} \log \cover(\epsilon_V, \mathfrak{R}_V, s_{\pr(V)})
		\end{align}
		
\end{theorem}
In Section~\ref{sec:computation_graph}~we formalize the notion of ``cutting'' nodes from the graph. The condition that node $V$'s value has norm at most $s_V$ is a simplification made for expositional purposes; our full Theorem~\ref{thm:covering_graph}~also applies if $O_{G}$ collapses to a constant whenever node $V$'s value has norm greater than $s_V$. This allows for the softened indicators $\one[\le s_i](\|h_i(x)\|)$ used in~\eqref{eq:tilde_z_loss}.

%% file: lipschitz_8pages.tex
\subsection{Lipschitz Augmentation of Computational Graphs} \label{subsec:lipschitzaug_8} 
The covering number bound of Theorem~\ref{thm:inf_covering_graph} relies on Lipschitzness w.r.t internal nodes of the graph under a worst-case choice of inputs. For deep networks, this can scale exponentially in depth via the product of weight norms and easily be larger than the average Lipschitz-ness over typical inputs. In this section, we explain a general operation to augment sequential graphs (such as neural nets) into graphs with better worst-case Lipschitz constants, so tools such as Theorem~\ref{thm:inf_covering_graph} can be applied. Formal definitions and theorem statements are in Section~\ref{sec:lipschitzaug}.

The augmentation relies on introducing terms such as the soft indicators in equation~\eqref{eq:z_prime_expr} and~\eqref{eq:tilde_z_loss} which condition on data-dependent properties. As outlined in Section~\ref{sec:overview}, they will translate to the data-dependent properties in the generalization bounds. We also require the augmented function to upper bound the original.  

We will present a generic approach to augment function compositions such as $z \triangleq \ell \circ f_k \circ \ldots  \circ f_1$, whose Lipschitz constants are potentially exponential in depth, with only properties involving the norms of the inter-layer Jacobians. We will produce $\tilde{z}$, whose worst-case Lipschitzness w.r.t. internal nodes can be polynomial in depth.

\textbf{Informal explanation of Lipschitz augmentation:} 
		In the same setting of Section~\ref{sec:overview}, recall that in~\eqref{eq:z_prime_expr}, our first unsuccessful attempt to smooth out the function was by multiplying indicators on the norms of the derivatives of the output: $\prod_{i=1}^k \one[\le \kappa_{i}] (\Opnorm{{\partial z}/{\partial h_i}})$.
	The difficulty lies in controlling the Lipschitzness of the new terms $\opnorm{\partial z/\partial h_i}$ that we introduce: by the chain rule, we have the expansion 
$
		\frac{\partial z}{\partial h_i} = \frac{\partial{z}}{\partial h_k} \frac{\partial h_k}{\partial h_{k - 1}} \cdots \frac{\partial h_{i + 1}}{\partial h_i}
$, where each $h_{j'}$ is itself a function of $h_j$ for $j' > j$. This means $\frac{\partial z}{\partial h_i}$ is a complicated function in the intermediate variables $h_{j}$ for $1 \le j \le k$. Bounding the Lipschitzness of $\frac{\partial z}{\partial h_i}$ requires accounting for the Lipschitzness of every term in its expansion, which is challenging and creates complicated dependencies between variables. 

Our key insight is that by considering a more complicated augmentation which conditions on the derivatives between all intermediate variables, we can still control Lipschitzness of the system, leading to the more involved augmentation presented in~\eqref{eq:tilde_z_loss}.
Our main technical contribution is Theorem~\ref{thm:inf_lip}, which we informally state below. 
\begin{theorem}[Informal version of Theorem~\ref{thm:lip}]\label{thm:inf_lip}The functions $\tilde{z}$ (defined in~\eqref{eq:tilde_z_loss}) can be computed by a family of computational graphs $\augG$ illustrated in Figure~\ref{fig:lipschitz_aug_8}. This family has internal nodes $V_i$ and $J_i$ computing $h_i$ and $Df_i[h_{i - 1}]$, respectively, and computes a modified output rule that augments the original with soft indicators. These soft indicators condition that the norms of the Jacobians and $h_i$ are bounded by parameters $\kappa_{j \ot i}, s_i$.
	
Importantly, the output $O_{\tilde{G}}$ is $\laug_{V_i}$, $\laug_{J_i}$-Lipschitz w.r.t. $V_i$, $J_i$, respectively, after cutting nodes $V_1, J_1, \ldots, V_{i - 1}, J_{i -1}$,  for parameters $\laug_{V_i}$, $\laug_{J_i}$ that are polynomials in $\kappa_{j \ot i}$, $s_i$. 
\end{theorem}
In addition, the augmented function $\tilde{z}$ will upper bound the original with equality when all the indicators are satisfied. 
The crux of the proof is leveraging the chain rule to decompose $\frac{\partial z}{\partial h_i}$ into a product and then applying a telescoping argument to bound the difference in the product by differences in individual terms. In Section~\ref{sec:lipschitzaug}~we present a formal version of this result and also apply Theorem~\ref{thm:inf_covering_graph} to produce a covering number bound for $\augG$.
		\begin{wrapfigure}{!h!}{0.35\textwidth}
\caption{Lipschitz augmentation (informally defined).}\label{fig:lipschitz_aug_8}
		\includegraphics[width=0.34\textwidth]{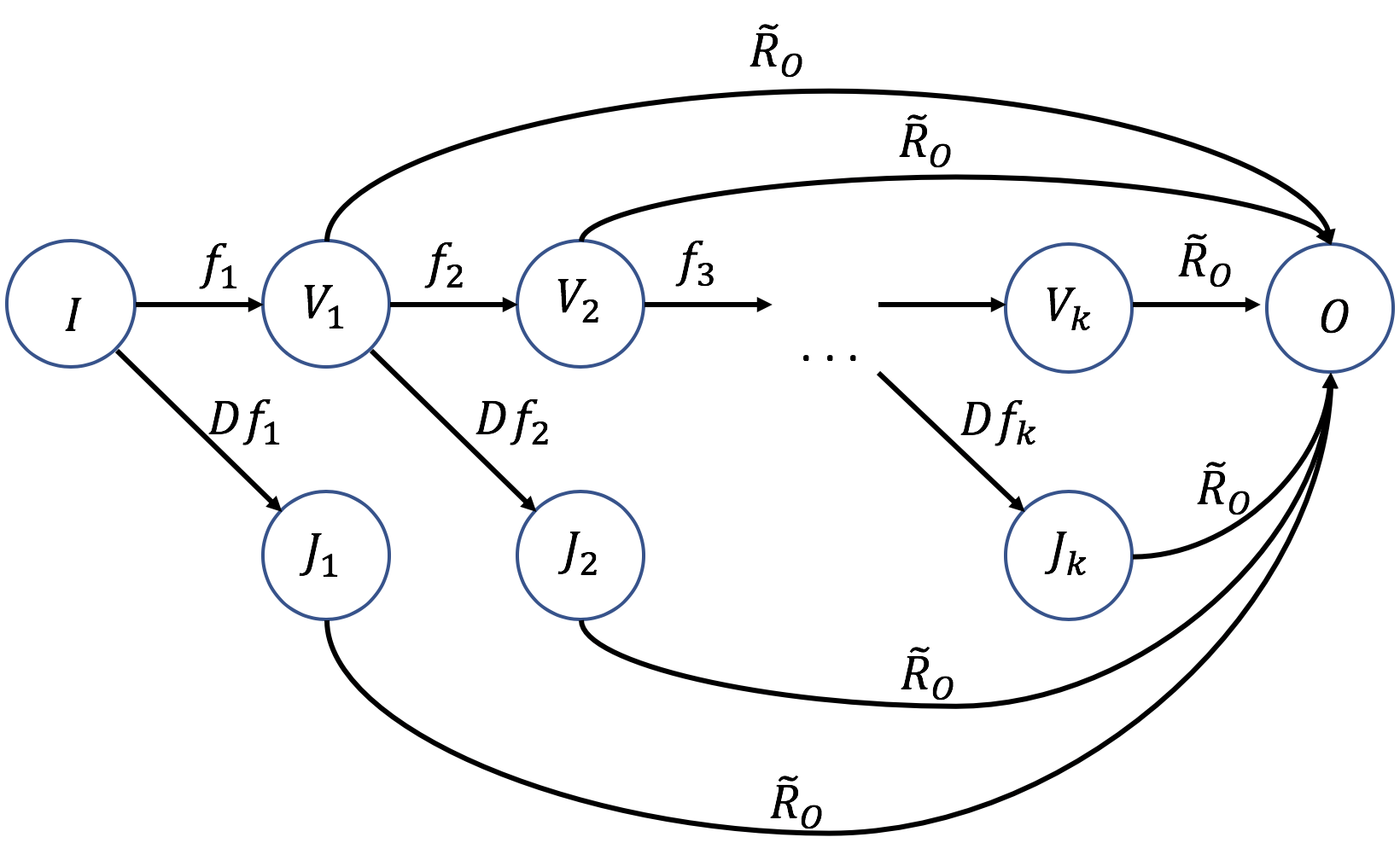}\\
	\end{wrapfigure}

%% file: graph_covering.tex
\section{Covering of Computational Graphs} \label{sec:computational_graph}
\label{sec:computation_graph}  
This section is a formal version of Section~\ref{sec:computation_graph_8}~with full definition and theorem statements. In this section, we adapt the notion of a computational graph to our setting. In Section~\ref{subsec:def_computation}, we formalize the notion of a computational graph and demonstrate how neural networks fit under this framework. In Section~\ref{subsec:covering_computation}, we define the notion of release-Lipschitzness that abstracts the sequential notion of Lipschitzness in Lemma~\ref{lem:composition-simple-k}. We show that when this release-Lipschitzness condition and a boundedness condition on the internal nodes hold, it is possible to cover a family of computational graphs by simply covering the function class at each vertex. 

\subsection{Formalization of computational graphs}\label{subsec:def_computation}
When we augment the neural network loss with data-dependent properties, we introduce dependencies between the various layers, making it complicated to cover the augmented loss. We use the notion of computational graphs to abstractly model these dependencies. 

Computational graphs are originally introduced by~\citet{bauer1974computational} to represent computational processes and study error propagation. 
Recall the notation $G(\cV,\cE,\{R_V\})$ introduced for a computational graph in Section~\ref{sec:computation_graph_8}, with input nodes $\cI_G= \{I_1,\dots, I_{p}\}$ and output node denoted by $O_G$. (It's straightforward to generalize to scenarios with multiple output nodes.)

For every variable $V\in \cV$, let $\cD_V$ be the space that $V$ resides in. 
If $V$ has $t$ direct predecessors $C_1,\dots, C_t$, then the associated composition rule $R_V$ is a function that maps  $\cD_{C_1}\otimes \cdots \otimes \cD_{C_t}$ to $\cD_{V}$. 
If $V$ is an input node, then the composition rule $R_V$ is not relevant. For any node $V$, the computational graph defines/induces a function that computes the variable $V$ from inputs, or in mathematical words, that maps the inputs space $\cD_{I_1}\otimes \cdots \otimes \cD_{I_p}$ to $\cD_V$. This associated function, denoted by $V$ again with slight abuse of notations, is defined recursively as follows: set $V(x_1,\dots, x_p)$ to
\begin{align}
\left\{\begin{array}{ll}
x_i & \textup{if $V$ is the $i$-th input node $I_i$}
\\
R_V(C_1(x_1,\dots, x_p), \dots, C_t(x_1,\dots, x_p)) & \textup{if $V$ has $t$ direct predecessors $C_1,\dots, C_t$}
\end{array}\right.\nonumber
\end{align}
More succinctly, we can write $V = R_V \circ (C_1\otimes \cdots \otimes C_t)$.  We also overload the notation $O_G$ to denote the function that the computational graph $G$ finally computes (which maps $\cD_{I_1}\otimes \cdots \otimes \cD_{I_p}$ to $\cD_{O}$). For any set $\cS = \{V_1, \ldots, V_t\} \subseteq \cV$, use $\cD_{\cS}$ to denote the space $\cD_{V_1} \otimes \cdots \otimes \cD_{V_t}$. We use $\pr(G, V)$ to denote the set of direct predecessors of $V$ in graph $G$, or simply $\pr(V)$ when the graph $G$ is clear from context. 

\begin{example} [Feed-forward neural networks] \label{ex:neuralnet}
	For an activation function $\phi$ and parameters $\{W^{(i)}\}$ we compute a neural net $F : \R^{d_I} \to \R^{d_O}$ as follows: $F(x) = W^{(\depthnn)} \phi( \cdots \phi(W^{(1)}x)\cdots)$.
Figure~\ref{fig:nncomputational}~depicts how this neural network fits into a computational graph with one input node, $2\depthnn - 1$ internal nodes, and a single output. Here we treat matrix operations and activations as distinct layers, and map each layer to a node in the computational graph. 
\end{example}
\begin{figure}
	\centering
	\includegraphics[width=0.7\textwidth]{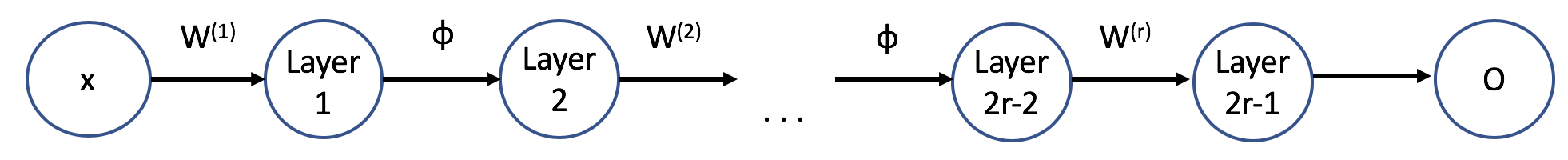}\\
	\caption{The computational graph corresponding to a neural network with $\depthnn$ weight matrices. Odd-indexed layers multiply matrices and even-indexed layers apply the activation $\phi$.} \label{fig:nncomputational}
\end{figure}
\subsection{Reducing graph covering to local function covering}\label{subsec:covering_computation}
In this section we introduce the notion of a family of computational graphs, generalizing the sequential family of function compositions in~\eqref{eq:sequential_family}. We define release-Lipschitzness, a condition which allows reduce covering the entire the graph family to covering the composition rules at each node. We formally state this reduction in Theorem~\ref{thm:covering_graph}.

\paragraph{Family of computational graphs: } Let $\cG = \{G(\cV, \cE, \{R_V\}):  \{R_V\} \in \mathfrak{R}\}$ be a family of computational graph with shared nodes and edges, where $\mathfrak{R}$ is a collection of lists of composition rules. This family of computational graphs defines a set of functions $O_\cG \triangleq \{O_G: G\in \cG\}$. 
We'd like to cover this set of functions in $O_\cG$ with respect to some metric $L(P_n, \gnorm{\cdot})$. 

For a list of composition rules $\{R_V\} \in \fR$ and subset $\cS \subseteq \cV$, we define the projection of composition rules onto $\cS$ by $\{R_V\}_{\cS} = \{R_V : V \in \cS\}$. Now let $\fR_S = \{\{R_V\}_{\cS} : \{R_V\} \in \fR\}$ denote the marginal collection of the composition rules on node subset $\cS$. 

For any computational graph $G$ and a non-input node $V \in \cV \setminus \cI$, we can define the following operation that ``releases'' $V$ from its dependencies on its predecessors by cutting all the inward edges: Let $\cut{G}{V}$ be sub-graph of $G$ where all the edges pointing towards $V$ are removed from the graph. Thus, by definition, $V$ becomes a new input node of the graph $\cut{G}{V}$: $\cI_{\cut{G}{V}} = \{V\}\cup \cI_G$. Moreover, we can ``recover'' the dependency by plugging the right value for $V$ in the new graph $\cut{G}{V}$: Let $V(x)$ be the function associated to the node $V$ in graph $G$, then we have
\begin{align}
\forall x \in \cD_{\cI}, ~~ O_{\cut{G}{V}}(V(x), x) = O_G(x) \label{eqn:relink}
\end{align}

In our proofs, we will release variables in orders. Let $\cS = (V_1,\dots, V_m)$ be an ordering of the intermediate variables $\cV\backslash(\cI\cup \{O\})$. We call $\cS$ a forest ordering if for any $i$, in the original graph $G$,  $V_i$ at most depends on the input nodes and $V_1,\dots, V_{i-1}$. For any sequence of variables $(V_1,\dots, V_t)$, we can define the graph obtained by releasing the variables in order: 
$
	\cut{G}{(V_1,\dots, V_t)} \triangleq \cut{(\cdots(\cut{G}{V_1}) \cdots)}{V_t}
$.
We next define the release-Lipschitz condition, which states that the graph function remains Lipschitz when we sequentially release vertices in a forest ordering of the graph. 
\begin{definition}[Release-Lipschitzness]\label{def:lip}
A graph $G$ is release-Lipschitz with parameters $\{\kappa_V\}$ w.r.t a forest ordering of the internal nodes, denoted by $(V_1,\dots, V_m)$ if the following happens: upon releasing $V_1,\dots, V_m$ in order from any $G\in \cG$, for any $0\le i \le m$,  we have that the function defined by the released graph $\cut{G}{(V_1,\dots, V_i)}$ is $\kappa_{V_i}$-Lipschitz in the argument $V_i$, for any values of the rest of the input nodes (=$\{V_1,\dots, V_{i-1}\}\cup \cI_G$.) We also say graph $G$ is release- Lipschitz if such a forest ordering exists. 
\end{definition}

Now we show that the release-Lipschitz condition allows us to cover any family of computational graphs whose output collapses when internal nodes are too large. The below is a formal and complete version of Theorem~\ref{thm:inf_covering_graph}. For the augmented loss defined in~\eqref{eq:tilde_z_loss}, the function output collapses to $1$ when internal computations are large. The proof is deferred to Section~\ref{app:computational_graph}. 
\begin{theorem}
	\label{thm:covering_graph}
Suppose $\cG$ is a computational graph with the associated family of lists of composition rules $\mathfrak{R}$, as formally defined above. 
Let $P_n$ be a uniform distribution over $n$ points in $\cD_\cI$. 
Let $\kappa_V$, $s_V$, and $\epsilon_V$ be three families of fixed parameters indexed by $\cV\backslash\cI$ (whose meanings are defined below). Assume the following: 
	\begin{itemize}
		\item[1.] Every $G\in \cG$ is release-Lipschitz with parameters $\{\kappa_V\}$ w.r.t a forest ordering of the internal nodes $(V_1,\dots, V_m)$ (the parameter $\kappa_V$'s and ordering doesn't depend on the choice of $G$.)
		\item[2.] For the same order as before, if $(v, x)\in (\cD_{V_1}\otimes \cdots \otimes \cD_{V_i})\otimes  \cD_\cI$ is an input of the released graph satisfying $\gnorm{v_j} \ge s_{V_j}$ for some $j \le i$, then $O_{\cut{G}{(V_1,\dots, V_i)}}(v,x) = c$ for some constant $c$. 
	\end{itemize}
	Then, small covering numbers for all of the local composition rules of $V$ with resolution $\epsilon_V$ would imply small covering number for the family of computational graphs with resolution $\sum_V \epsilon_V\kappa_V$:
		\begin{align}
		\label{eq:covering_graph-1}
		\log \cover(\sum_{V\in \cV \setminus \cI \cup \{O\}} \kappa_V \epsilon_V + \epsilon_O, O_\cG, s_{\cI}) \le   \sum_{V\in \cV \setminus \cI} \log \cover(\epsilon_V, \mathfrak{R}_{\{V\}}, s_{\pr(V)})
		\end{align}
		
\end{theorem}

%% file: lipschitz.tex
\section{Lipschitz Augmentation of Computational Graphs}
 \label{sec:lipschitzaug} In this section, we provide a more thorough and formal presentation of the augmentation framework of Section~\ref{subsec:lipschitzaug_8}. 
 
 The covering number bound for the computational graph family $\cG$ in Theorem~\ref{thm:covering_graph} relies on the release-Lipschitzness condition (condition 1 of Theorem~\ref{thm:covering_graph}) and rarely holds for deep computational graphs such as deep neural networks. The conundrum is that the worst-case Lipschitzness as required in the release-Lipschitz condition\footnote{We say the Lipschitzness required is worst case because the release-Lipschitz condition requires the Lipschitzness of nodes for any possible choice of inputs} is very likely to scale in the product of the worst-case Lipschitzness of each operations in the graph, which can easily be exponentially larger than the average Lipschitzness over typical examples.  

In this section, we first define a model of sequential computational graphs, which captures the class of neural networks. Before Lipschitz augmentation, the worst-case Lipschitz constant of graphs in this family could scale exponentially in the depth of the graph. In Definition~\ref{def:aug}, we generalize the operation of~\eqref{eq:tilde_z_loss}~to augment any family $\cG$ of sequential graphs and produce a family $\augG$ satisfying the release-Lipschitz condition. In Theorem~\ref{thm:sequential_covering}, we combine this augmentation with the framework of~\ref{thm:covering_graph} to produce general covering number bounds for the augmented graphs. For the rest of this section we will work with sequential families of computational graphs. 

A sequential computational graph has nodes set $\cV = \{I, V_1, \ldots, V_q, O\}$, where $I$ is the single input node, and all the edges are $\cE = \{(I,V_1), (V_1,V_2),\cdots, (V_{q-1},V_q)\}\cup \{(V_1,O), \dots, (V_q,O)\}$. We often use the notation $V_0$ to refer to the input $I$. Below we formally define the augmentation operation.

	\begin{definition}[Lipschitz augmentation of sequential graphs] \label{def:aug}Given a differentiable sequential computational graph $G$ with $q$ internal nodes $V_1,\dots, V_q$, define its Lipschitz augmentation $\tildeG$ as follows. We first add $q$ nodes to the graph denoted by $J_1,\dots, J_q$. The composition rules for original internal nodes remain the same, and the composition rule for $J_i$ is defined as 
		\begin{align*}
		\tildeR_{J_i} = DR_{V_i}
		\end{align*}
		Here $DR_{V_i}$ is the total derivative of the function $R_{V_i}$. In other words, the variable $J_i$ is a Jacobian for $R_{V_i}$, a linear operator that maps $\cD_{V_{i-1}}$ to $\cD_{V_i}$. (Note that if $V_i$'s are considered as vector variables, then $J_i$'s are matrix variables.) We equip the space of $J_i$ with operator norm, denoted by $\opnorm{\cdot}$, induced by the original norms on spaces $V_{i-1}$ and $V_i$. The Lipschitz-ness w.r.t variable $J_i$ will be measured with operator norm.
		
		We pre-determine a family of parameters $\kappa_{j\ot i}$ for all pairs $(i,j)$ with $i \le j$. The final loss is augmented by a product of soft indicators that truncates the function when any of the Jacobians is much larger than $\kappa_{i\ot j}$ :
		\begin{align*}
		\tildeR_{O}(x, v_1, \ldots, v_\depthlocal, D_1, \ldots, D_\depthlocal)  \triangleq (R_{O}(x, v_1, \ldots, v_\depthlocal) - 1)\prod_{i \le j} \one[\le \kappa_{j \ot i}] (\opnorm{D_j \cdots D_i})+ 1
		\end{align*}
		where $x \in \cD_{\cI}$, $v_i \in \cD_{V_i}$, and $D_i \in \cD_{J_i}$. Note that $D_j \cdots D_i$ is the total derivative of $V_j$ w.r.t $V_i$, and thus the $\kappa_{j \ot i}$ has the interpretation as an intended bound of the Jacobian between pairs of layers (variables). Figure~\ref{fig:computation_graph}~depicts the augmentation. 
	\end{definition}

	Note that under these definitions, we finally get that the output function of $\tildeG$ computes
	\begin{align}
	O_{\tildeG}(x) = (O_G(x) - 1) \prod_{i \le j} \one[\le \kappa_{j \ot i}] (\opnorm{D V_j(x) \cdots D V_i(x)}) + 1 \label{eq:Otilde}
	\end{align}
	which matches~\eqref{eq:tilde_z_loss}~for the example in Section~\ref{sec:overview}.
	We note that the graph $\tilde{G}$ contains the original $G$ as a subgraph. Furthermore, by Claim~\ref{claim:indicator_stack}, $O_{\tilde{G}}$ upper bounds $O_{G}$, which is desirable when $G$ computes loss functions. The below theorem, which formalizes Theorem~\ref{thm:inf_lip}, proves release-Lipschitzness for $\augG$.
\begin{theorem}
	\label{thm:lip}[Lipschitz guarantees of augmented graphs]
	Let $\cG$ be a family of sequential computational graphs. Suppose for any $G\in \cG$, the composition rule of the output node, $R_{O_G}$, is $c_i$-Lipschitz in variable $V_i$ for all $i$, and it only outputs value in $[0,1]$. 
	Suppose that $DR_{V_i}$ is $\ulip_i$-Lipschitz for each $i$.\footnote{Note that $DR_{V_i}$ maps a vector in space $\cD_{V_{i-1}}$ to an linear operator that maps $\cD_{V_{i-1}}$ to $\cD_{V_i}$.} 
	Let $\kappa_{j \ot i}$ (for $i\le j$) be a set of parameters that we intend to use to control Jacobians in the Lipschitz augmentation. With them, we apply Lipschitz augmentation as defined in Definition~\ref{def:aug} to every graph in $\cG$ and obtain a new family of graphs, denoted by $\widetilde{\cG}$. 
	
	Then, the augmented family $\widetilde{\cG}$ is release-Lipschitz  (Definition~\ref{def:lip}) with parameters $\tilde{\kappa}_V$'s below:
{
\begin{align*}\laug_{V_i} &\triangleq \sum_{i \le j \le \depthlocal} 3c_{j} \kappa_{j \ot i + 1} +18 \sum_{1 \le j \le j' \le \depthlocal} \sum_{i' = \max\{i + 1, j\}}^{j'} \frac{\ulip_{i'}\kappa_{j' \ot i' + 1} \kappa_{i' - 1 \ot i + 1} \kappa_{i' - 1 \ot j}}{\kappa_{j' \ot j}},~~\\ \laug_{J_i} &\triangleq \sum_{j \le i \le j'} \frac{4\kappa_{j' \ot i + 1} \kappa_{i - 1 \ot j}}{\kappa_{j' \ot j}}\end{align*}}
where for simplicity in the above expressions, we extend the definition of $\kappa$'s to $\kappa_{j -1\ot j} = 1$. 

\end{theorem}

		\begin{wrapfigure}{!h!}{0.40\textwidth}
\caption{Lipschitz augmentation (formally defined).}\label{fig:computation_graph}
		\includegraphics[width=0.38\textwidth]{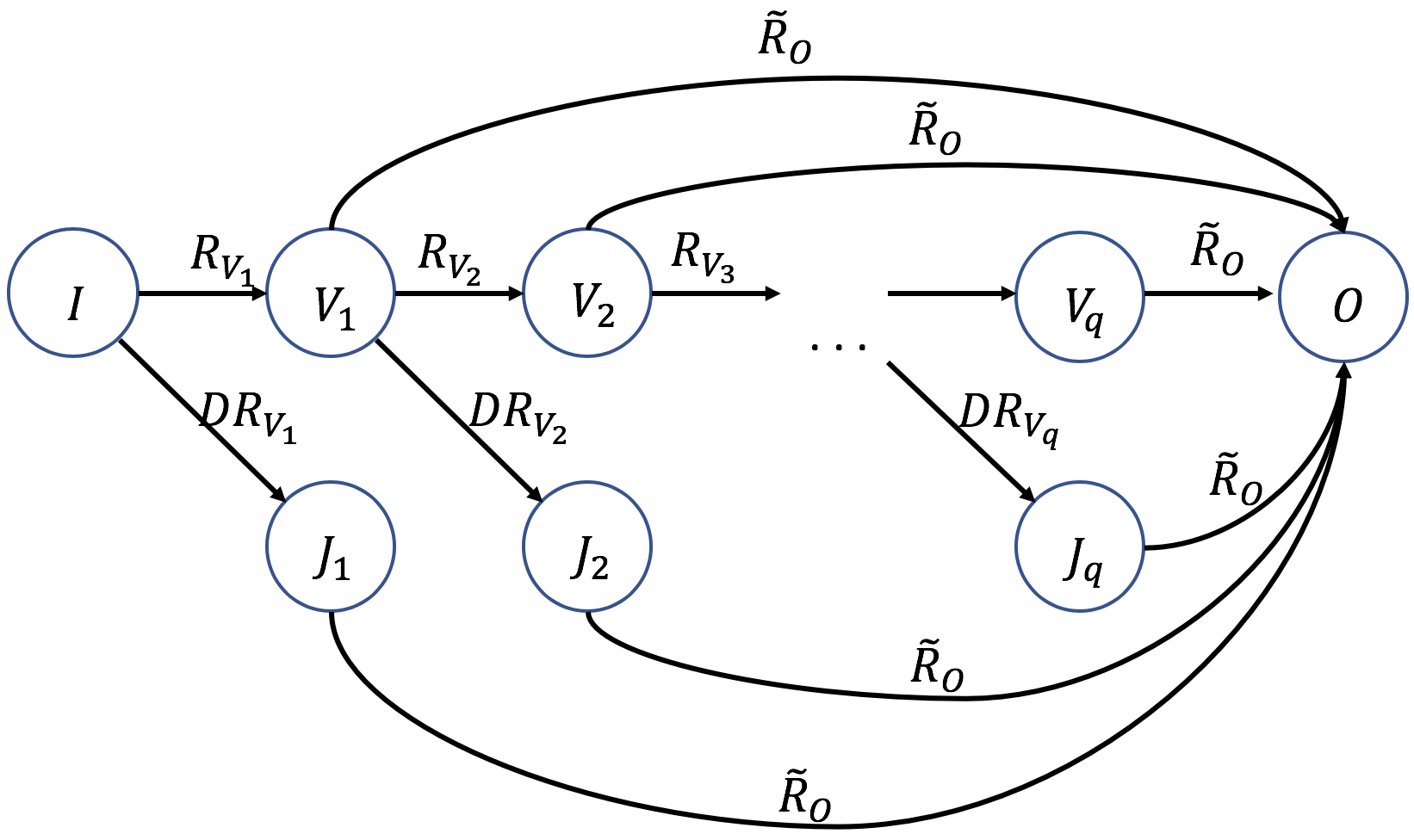}\\
	\end{wrapfigure}
	Finally, we combine Theorems~\ref{thm:covering_graph} and Theorems~\ref{thm:lip}~to derive covering number bounds for any Lipschitz augmentation of sequential computational graphs. The final covering bound in~\eqref{eq:sequential_covering-bound}~can be easily computed given covering number bounds for each individual function class. In Section~\ref{sec:neural_net_main}, we use this theorem to derive Rademacher complexity bounds for neural networks. The proof is deferred to Section~\ref{app:augmentation}. In Section~\ref{sec:recurrent}, we also use these tools to derive Rademacher complexity bounds for RNNs.
\begin{theorem}
	\label{thm:sequential_covering}
		Consider any family $\cG$ of sequential computational graphs satisfying the conditions of Theorem~\ref{thm:lip}. By combining the augmentation of Definition~\ref{def:aug}~with additional indicators on the internal node norms, we can construct a new family $\augG$ of computational graphs which output 
	\begin{align*}
		O_{\tildeG}(x) = (O_G(x) - 1) \prod_{i = 1}^\depthlocal \one[\le s_{V_i}] (\|V_i(x)\|) \prod_{1 \le i \le j \le \depthlocal} \one[\le \kappa_{j \ot i}] (\opnorm{DV_j(x) \cdots DV_i(x)}) + 1
	\end{align*}
	The family $\augG$ satisfies the following guarantees: 
	\begin{enumerate}
		\item Each computational graph in $\augG$ upper bounds its counterpart in $\cG$, i.e. $O_{\tildeG}(x) \ge O_G(x)$.
		\item Define 
$
		\laug_{V_i}' \triangleq \laug_{V_i} + \sum_{i \le j \le \depthlocal} s_{V_j}^{-1}\cdot \kappa_{j \ot i + 1}
$
				and $\laug_{J_i}' = \laug_{J_i}$ where $\laug_{V_i}, \laug_{J_i}$ are defined as in Theorem~\ref{thm:lip}. 
				Then for any node-wise errors $\{\epsilon_V\}$,\\
		\resizebox{0.9\textwidth}{!}{
			\begin{minipage}{0.9\textwidth}
		\begin{align}
&			\log \cover(\sum_{i \ge 1} \laug_{V_i}'\epsilon_{V_i} + \laug_{J_i}\epsilon_{J_i} + \epsilon_O, O_{\augG}, s_{\cI}) \label{eq:sequential_covering-bound}\\
			& \le \sum_{i \ge 1} \log \cover(\epsilon_{V_i}, \fR_{V_i}, 2s_{V_{i - 1}}) + \log \cover(\epsilon_{J_i}, D\fR_{V_i}, 2s_{V_{i - 1}}) + \log \cover(\epsilon_O, \fR_O, \{2s_{V_j}\}_{j = 1}^q \cup \{I\})\nonumber
		\end{align}
		\end{minipage}
		}\\
	
		where $D\fR_{V_i}$ denotes the family of total derivatives of functions in $\fR_{V_i}$ and $V_0$ the input vertex.	\end{enumerate}
\end{theorem}

%% file: neural_net_main.tex
\section{Application to Neural Networks} \label{sec:neural_net_main}
In this section we provide our generalization bound for neural nets, which was obtained using machinery from Section~\ref{sec:computation_graph_8}. Define a neural network $F$ parameterized by $r$ weight matrices $\{W^{(i)}\}$ by $F(x) = W^{(r)} \phi(\cdots \phi(W^{(1)}(x)) \cdots)$. We use the convention that activations and matrix multiplications are treated as distinct layers indexed with a subscript, with odd layers applying a matrix multiplication and even layers applying $\phi$ (see Example~\ref{ex:neuralnet}~for a visualization). Additional notation details and the proof are in Section~\ref{sec:neural_net_app}.

The below result follows from modeling the neural net loss as a sequential computational graph and using our augmentation procedure to make it Lipschitz in its nodes with parameters $\kappah{i},\kappaj{i}$. Then we cover the augmented loss to bound its Rademacher complexity.
\begin{theorem}
	\label{thm:gen_union_bound}
	Assume that the activation $\phi$ is 1-Lipschitz with a $\bar{\sigma}_\phi$-Lipschitz derivative. Fix reference matrices $\{A^{(i)}\}$, $\{B^{(i)}\}$. With probability $1 - \delta$ over the random draws of the data $P_n$, all neural networks $F$ with parameters $\{W^{(i)}\}$ and positive margin $\gamma$ satisfy:
	{\small
	 \begin{align*}
	 	\Exp_{(x, y) \sim P} [l_{\textup{0-1}}(F(x), y)] \le \tilde{O}\left(\frac{\left(\sum_{i} (\kappah{i}a^{(i)} t^{(i - 1)})^{2/3} + (\kappaj{i}b^{(i)})^{2/3}\right)^{3/2}}{\sqrt{n}} + r\sqrt{\frac{ \log(1/\delta)}{n}}\right)
	 \end{align*}}
	 where 
$\kappaj{i} \triangleq \sum_{1 \le j \le 2i - 1 \le j' \le 2\depthnn - 1} \frac{\sigma_{j' \ot 2i} \sigma_{2i - 2 \ot j}}{\sigma_{j' \ot j}}$, and 
$
	\kappah{i} \triangleq \xi +\frac{\sigma_{2r - 1 \ot 2i}}{\gamma} + \sum_{i \le i' < r} \frac{\sigma_{2i' \ot 2i}}{t^{(i')}} + \sum_{1 \le j \le j' \le 2\depthnn - 1} \sum_{\substack{j'' = \max\{2i, j\},\\ j'' \textup{ even }}}^{j'} \frac{\bar{\sigma}_\phi \sigma_{j'\ot j'' + 1} \sigma_{j'' - 1 \ot 2i} \sigma_{j'' - 1 \ot j}}{\sigma_{j' \ot j}}$. 
	
	In these expressions, we define $\sigma_{j - 1 \ot j} = 1$, $\xi = \poly(r)^{-1}$, and:
 	\[a^{(i)} \triangleq \|{W^{(i)}}^\top - {A^{(i)}}^\top\|_{2, 1} + \xi, b^{(i)} \triangleq \|{W^{(i)}} - {B^{(i)}}\|_{1, 1} + \xi\]
	\[t^{(0)} \triangleq \max_{x  \in P_n} \|x\| + \xi, \ t^{(i)} \triangleq \max_{x \in P_n} \|F_{2i \ot 1}(x)\| + \xi\]
	 \[\sigma_{j' \ot j} \triangleq \max_{x \in P_n} \opnorm{Q_{j' \ot j}(x)} + \xi, \textup{ and } \gamma \triangleq \min_{(x, y) \in P_n} [F(x)]_y - \max_{y' \ne y} [F(x)]_{y'} > 0\] 
	where $Q_{j' \ot j}$ computes the Jacobian of layer $j'$ w.r.t. layer $j$. Note that the training error here is $0$ because of the existence of positive margin $\gamma$.
\end{theorem}
We note that our bound has no explicit dependence on width and instead depends on the $\|\cdot \|_{2, 1}, \|\cdot \|_{1, 1}$ norms of the weights offset by reference matrices $\{A^{(i)}\}, \{B^{(i)}\}$. These norms can avoid scaling with the width of the network if the difference between the weights and reference matrices is sparse. The reference matrices $\{A^{(i)}\}, \{B^{(i)}\}$ are useful if there is some prior belief before training about what weight matrices are learned, and they also appear in the bounds of~\citet{bartlett2017spectrally}. In Section~\ref{sec:recurrent}, we also show that our techniques can easily be extended to provide generalization bounds for RNNs scaling polynomially in depth via the same quantities $t^{(i)}, \sigma_{j' \ot j}$.

%% file: experiments.tex
\section{Experiments}\label{sec:experiments}
Though the main purpose of the paper is to study the data-dependent generalization bounds from a theoretical perspective, we provide preliminary experiments demonstrating that the proposed complexity measure and generalization bounds are empirically relevant. We show that regularizing the complexity measure leads to better test accuracy. Inspired by Theorem~\ref{thm:gen_union_bound}, we directly regularize the Jacobian of the classification margin w.r.t outputs of normalization layers and after residual blocks. Our reasoning is that normalization layers control the hidden layer norms, so additionally regularizing the Jacobians results in regularization of the product, which appears in our bound. We find that this is effective for improving test accuracy in a variety of settings. We note that \citet{sokolic2017robust} show positive experimental results for a similar regularization technique in data-limited settings. 

Suppose that $m(F(x), y)= [F(x)]_y - \max_{j\ne y}[t]_j$ denotes the margin of the network for example $(x, y)$. Letting $h^{(i)}$ denote some hidden layer of the network, we define the notation 
$
J^{(i)} \triangleq \frac{\partial}{\partial h^{(i)}}m(F(x),y) 
$
and use training objective
\begin{align*}
\hat{L}_{\textup{reg}}[F] \triangleq \E_{(x,y)\sim P_n} \left[l(x,y) + \lambda \left(\sum_{i}\1(\|J^{(i)}(x)\|_F^2 \ge \sigma)\|J^{(i)}(x)\|_F^2\right) \right]
\end{align*}
where $l$ denotes the standard cross entropy loss, and $\lambda, \sigma$ are hyperparameters. Note the Jacobian is taken with respect to a scalar output and therefore is a vector, so it is easy to compute. 

For a WideResNet16 \citep{zagoruyko2016wide} architecture, we train using the above objective. The threshold on the Frobenius norm in the regularization is inspired by the truncations in our augmented loss (in all our experiments, we choose $\sigma= 0.1$). We tune the coefficient $\lambda$ as a hyperparameter. In our experiments, we took the regularized indices $i$ to be last layers in each residual block as well as layers in residual blocks following a BatchNorm in the standard WideResNet16 architecture. In the LayerNorm setting, we simply replaced BatchNorm layers with LayerNorm. The remaining hyperparameter settings are standard for WideResNet; for additional details see Section~\ref{sec:exp_details}.

Figure~\ref{fig:jreg}~shows the results for models trained and tested on CIFAR10 in low learning rate and no data augmentation settings, which are settings where generalization typically suffers. We also experiment with replacing BatchNorm layers with LayerNorm and additionally regularizing the Jacobian. We observe improvements in test error for all these settings. In Section~\ref{sec:empirical_comparison}, we empirically demonstrate that our complexity measure indeed avoids the exponential scaling in depth for a WideResNet model trained on CIFAR10. 

\begin{table}
\centering
	\caption{Test error for a model trained on CIFAR10 in various settings.}
\label{fig:jreg}
			\begin{tabular}{c  c  c  c} 				Setting & Normalization & Jacobian Reg & Test Error\\
				\cline{1-4}
				Baseline & BatchNorm & $\times$ & 4.43\%\\
				\cline{1-4}
				\multirow{2}{*}{Low learning rate (0.01)}& \multirow{2}{*}{BatchNorm} & $\times$ & 5.98\%\\
				 & & $\checkmark$ &  \textbf{5.46}\%\\
				\cline{1-4}
				\multirow{2}{*}{No data augmentation} & \multirow{2}{*}{BatchNorm} & $\times$ & 10.44\%\\
				 & & $\checkmark$ & \textbf{8.25\%}\\
				\cline{1-4}
				\multirow{3}{*}{No BatchNorm} & None & $\times$ & 6.65\%\\
				\cline{2-3}
				&\multirow{2}{*}{LayerNorm~\citep{ba2016layer}} & $\times$ & 6.20\%\\
				&  & $\checkmark$ & \textbf{5.57\%}
			\end{tabular}
\end{table}

%% file: conclusion.tex
\section{Conclusion}
In this paper, we tackle the question of how data-dependent properties affect generalization. We prove tighter generalization bounds that depend polynomially on the hidden layer norms and norms of the interlayer Jacobians. To prove these bounds, we work with the abstraction of computational graphs and develop general tools to augment any sequential family of computational graphs into a Lipschitz family and then cover this Lipschitz family. This augmentation and covering procedure applies to any sequence of function compositions. An interesting direction for future work is to generalize our techniques to arbitrary computational graph structures. In follow-up work~\citep{wei2019improved}, we develop a simpler technique to derive Jacobian-based generalization bounds for both robust and clean accuracy, and we present an algorithm inspired by this theory which empirically improves performance over strong baselines.

\section*{Acknowledgments}
CW was supported by a NSF Graduate Research Fellowship. Toyota Research Institute (TRI) provided funds to assist the authors with their research but this article solely reflects the opinions and conclusions of its authors and not TRI or any other Toyota entity.

%% file: neural_net_application.tex
\section{Missing Proofs for Section~\ref{sec:neural_net_main}}\label{sec:neural_net_app}
We first elaborate more on the notations introduced in Section~\ref{sec:neural_net_main}. First, by our indexing, matrix $W^{(i)}$ will be applied in layer $2i - 1$ of the network, and even layers $2i$ apply $\phi$. We let $F_{j' \ot j}$ denote the function computed between layers $j$ and $j'$ and $Q_{j' \ot j} = DF_{j' \ot j} \circ F_{j' -1 \ot 1}$ denote the layer $j$-to-$j'$ Jacobian. By our definition of $F_{j' \ot j}$, $F_{2j \ot 2j} = \phi$, $F_{2j - 1 \ot 2j - 1} = h \mapsto W^{(j)}h$, and $F_{j' \ot j}$ is recursively computed by $F_{j' \ot j'} \circ F_{j' - 1 \ot j}$ for $j' > j$. We will use the convention that $F_{j - 1 \ot j}$ computes the identity mapping for $i \le j$. 

$P$ will denote a test distribution over examples $x$ and labels $y$, and $P_n$ will denote the distribution on training examples.

For a class of real-valued functions $\cL$ and dataset $P_n$, define the empirical Rademacher complexity of this function class by 
\begin{align}
\label{eq:rademacher}
\rad(\cL) = \frac{1}{n}\Exp_{\alpha_i}\left[\sup_{l \in \cL} \sum_{i} \alpha_i l(x_i)\right]
\end{align} 
where $\alpha_i$ are independent uniform $\pm 1$ random variables. Let $m(t, y) \triangleq [t]_y - \max_{j\ne y}[t]_j$ denote the margin operator for label $y$, and $l_{\gamma}(t, y)\triangleq \one(m(t, y) \le 0) - \one(0< m(t, y) \le \gamma)\cdot m(t, y)/\gamma$ denote the standard ramp loss, which is $1/\gamma$-Lipschitz. We will work in the neural network setting defined in Section~\ref{sec:neural_net_main}. We will first state our generalization bound for neural networks. 

\begin{theorem}\label{thm:gen_no_union_bound}
	Assume that the activation $\phi$ is $1$-Lipschitz with $\bar{\sigma}_{\phi}$-Lipschitz derivative. Fix parameters $\sigma_{j' \ot j}$, $t^{(i)}$, $a^{(i)}$, $b^{(i)}$, $\gamma$ and reference matrices $\{A^{(i)}\}$, $\{B^{(i)}\}$. With probability $1 - \delta$ over the random draws of the distribution $P_n$, all neural networks $F$ with parameters $\{W^{(i)}\}$ satisfying the following data-dependent conditions: 
	\begin{enumerate}
		\item Hidden layers norms are controlled: $\max_{x \in P_n}\|F_{2i \ot 1}(x)\| \le t^{(i)} \ \forall 1 \le i \le \depthnn$. 
		\item Jacobians are balanced: $\max_{x \in P_n} \opnorm{Q_{j' \ot j}(x)} \le \sigma_{j' \ot j} \ \forall j < j'$. 
		\item The margin is large: $\min_{(x, y) \in P_n} [F(x)]_y - \max_{y' \ne y} [F(x)]_{y'} \ge \gamma > 0$.
	\end{enumerate}
	and the additional data-independent condition $$\|{W^{(i)}}^\top - {A^{(i)}}^\top\|_{2, 1} \le a^{(i)}, \|{W^{(i)}}- {B^{(i)}}\|_{1, 1} \le b^{(i)}, \opnorm{W^{(i)}} \le \sigma_{2i - 1 \ot 2i - 1}$$
	will have the following generalization to test data: 
	{	\begin{align*}
	\Exp_{(x, y) \sim \cD}[l_{\textup{0-1}}(F(x), y)] \le \tilde{O}\left(\frac{\left(\sum_{i} (\kappah{i} a^{(i)} t^{(i - 1)})^{2/3} + (\kappaj{i}b^{(i)})^{2/3}\right)^{3/2}}{\sqrt{n}}\right) + \sqrt{\frac{\log(1/\delta)}{n}}
	\end{align*}}
	where 
	\begin{align}
	\kappaj{i} \triangleq \sum_{1 \le j \le 2i - 1 \le j' \le 2\depthnn - 1} \frac{4\sigma_{j' \ot 2i} \sigma_{2i - 2 \ot j}}{\sigma_{j' \ot j}}\label{eq:kappa_j}
	\end{align}
	\begin{align}
	\begin{split}
	\kappah{i} \triangleq &\frac{\sigma_{2r - 1 \ot 2i}}{\gamma} + \sum_{i \le i' < r} \frac{3\sigma_{2i' \ot 2i}}{t^{(i')}} \\&+ \sum_{1 \le j \le j' \le 2\depthnn - 1} \sum_{\substack{j'' = \max\{2i, j\},\\ j'' \textup{ even }}}^{j'} \frac{\bar{\sigma}_\phi \sigma_{j'\ot j'' + 1} \sigma_{j'' - 1 \ot 2i} \sigma_{j'' - 1 \ot j}}{\sigma_{j' \ot j}}\label{eq:kappa_h}	\end{split}
	\end{align}
	Here we use the convention that $\sigma_{j - 1 \ot j} = 1$ and let $t^{(0)} = \max_{x \in P_n} \|x\|$. 
	\end{theorem}

This generalization bound follows straightforwardly via the below Rademacher complexity bound for the augmented loss class: 
\begin{theorem}
	\label{thm:nnrad}
	Suppose that $\phi$ is $1$-Lipschitz with $\bar{\sigma}_{\phi}$-Lipschitz derivative. Define the following class of neural networks with norm bounds on its weight matrices with respect to reference matrices $\{A^{(i)}\}, \{B^{(i)}\}$: 
	\begin{align*}
	\cF \triangleq \left\{x \mapsto F(x) : \|{W^{(i)}}^\top - {A^{(i)}}^\top\|_{2, 1} \le a^{(i)}, \|{W^{(i)}} - {B^{(i)}}\|_{1, 1} \le b^{(i)}, \opnorm{W^{(i)}} \le \sigma^{(i)} \right\}
	\end{align*}
	Fix parameters $t^{(i)}$ and $\sigma_{j' \ot j}$ for $j' \ge j$ with
	$\sigma_{2i \ot 2i} = 1$ and $\sigma_{2i - 1 \ot 2i - 1} = \sigma^{(i)}$. When we apply this theorem, we will choose $\sigma_{j' \ot j}$ and $t^{(i)}$ which upper bound the layer $j$ to $j'$ Jacobian norm and $i$-th hidden layer norm, respectively. Define the class of augmented losses 
	\begin{align*}
	\cL_{\textup{aug}} \triangleq \left\{(l_{\gamma} - 1)\circ F \prod_{i = 1}^{\depthnn - 1} \one[\le t^{(i)}](\|F_{2i \ot 1}\|) \prod_{1 \le j < j' \le 2\depthnn - 1} \one[\le \sigma_{j' \ot j}] (\opnorm{Q_{j' \ot j}}) + 1: F \in \cF\right\}
	\end{align*}
	and define for $1 \le i \le \depthnn$, $\kappaj{i}, \kappah{i}$ meant to bound the influence of the matrix $W^{(i)}$ on the Jacobians and hidden variables, respectively as in~\eqref{eq:kappa_j},~\eqref{eq:kappa_h}.
		Then we can bound the empirical Rademacher complexity of the augmented loss class by
	\begin{align*}
	\rad(\cL_{\textup{aug}}) = \tilde{O}\left(\frac{\left(\sum_i (\kappah{i} a^{(i)} t^{(i- 1)})^{2/3} + (\kappaj{i}b^{(i)})^{2/3}\right)^{3/2}}{\sqrt{n}}\right) 
	\end{align*}
	where we recall that the notation $\tilde{O}$ hides log factors in the arguments and the dimension of the weight matrices. 
\end{theorem}
\begin{proof}
	We associate the un-augmented loss class on neural networks $l_{\gamma} \circ \cF$ with a family of sequential computation graphs $\cG$ with depth $2\depthnn - 1$. The composition rules are as follows: for internal node $V_{2i}$, $\fR_{V_{2i}} = \{\phi\}$, the set with only one element: the activation $\phi$. We also let $\fR_{V_{2i - 1}} = \{h \mapsto Wh : \|W^\top - {A^{(i)}}^\top\|_{2, 1} \le a^{(i)}, \|W - {B^{(i)}}\|_{1, 1} \le b^{(i)}, \opnorm{W} \le \sigma^{(i)}\}$. Finally, we choose $\fR_O$ to be the singleton class $\{l_\gamma\}$. Our collection of computation rules is then simply $\fR = \fR_{V_1}\otimes \cdots \otimes \fR_{V_{2\depthnn  -1}} \otimes \fR_{O}$. Since $O_\cG$ takes values in $[0, 1]$, we can apply Theorem~\ref{thm:sequential_covering} on this class $\cG$ using $s_\cI = \max_{x \in P_n} \|x\|$, $s_{V_{2i}} = t^{(i)}$, $s_{V_{2i - 1}} = \infty$, $\kappa_{2i \ot 2i} = 1$, $\kappa_{2i - 1 \ot 2i - 1} = \sigma^{(i)}$, and $\kappa_{j' \ot j} =  \sigma_{j' \ot j}$ for $j' > j$. Furthermore, we note that $\ulip_{2i} = \bar{\sigma}_{\phi}$, and $\ulip_{2i - 1} = 0$ as the Jacobian is constant for matrix multiplications. We thus obtain the class $\augG$ where each augmented loss upper bounds the corresponding loss in $\cG$. Recall that $J_i$ denote the additional nodes in our augmented computation graph. Note that under these choices of $s_{V_{2i - 1}}$, $\kappa_{i \ot i}$, we get that
	\begin{align}
	\one[\le \kappa_{2i \ot 2i}](\opnorm{J_{2i}(x)}) &= \one[\le 1](\opnorm{D\phi \circ V_{2i - 1}(x)}) = 1 \tag{as $|\phi'| \le 1$}\\
	\one[\le \kappa_{2i - 1 \ot 2i - 1}](\opnorm{J_{2i - 1}(x)}) &= \one [\le \sigma^{(i)}] (\opnorm{W^{(i)} \circ V_{2i - 2}(x)}) = 1 \tag{as $W^{(i)} \le \sigma^{(i)}$}\\
	\one[\le s_{V_{2i - 1}}](\|V_{2i - 1}(x)\|) &= \one[\le \infty](\|V_{2i - 1}(x)\|) = 1 \notag
	\end{align}
	Furthermore, the other indicators in the augmented loss map to indicators in the outputs of our augmented graphs $O_{\tilde{G}}$, so therefore the families $\cL_{\textup{aug}}$ defined in the theorem statement and $\augG$ are equivalent. Thus, it suffices to bound the Rademacher complexity of $\augG$. To do this, we invoke covering numbers. By Theorem~\ref{thm:sequential_covering}, we bound the covering number of $O_{\augG}$:
	\begin{align}
	\begin{split}\label{eq:nnrad-1}
	\log \cover(\sum_{i \ge 1} (\laug_{V_i} + \laug_{J_i})\epsilon_V + \epsilon_O, O_{\augG}, s_{\cI}) \le \\\sum_{i \ge 1} \log \cover(\epsilon_{V_i}, \fR_{V_i}, 2s_{V_{i - 1}}) + \log \cover(\epsilon_{J_i}, D\fR_{V_i}, 2s_{V_{i - 1}}) + \log \cover(\epsilon_O, \fR_O, \{2s_{V_i}\}_{i \ge 0})
	\end{split}
	\end{align}
	where $\laug_{V_i}$, $\laug_{J_i}$ are defined in the statement of Theorem~\ref{thm:sequential_covering}. After plugging in our values for $\ulip_j$, $s_{V_j}$, $\kappa_{j' \ot j}$ in our application of Theorem~\ref{thm:sequential_covering} and noting that $c_{2i} = 1/t^{(i)}$, $c_{2i - 1} = 0$ for $i < r$ and $1/\gamma$ for $i = r$ (as the margin loss is $1/\gamma$-Lipschitz), we obtain that
	\begin{align*}
	\laug_{V_{2i - 1}} = \kappah{i}, 
	\laug_{J_{2i - 1}} = \kappaj{i}
	\end{align*}
	We first note that the last term in~\eqref{eq:nnrad-1} is simply 0 because there is exactly one output function in $\fR_O$. Now for the other terms of~\eqref{eq:nnrad-1}: by definition $\fR_{V_{2i}}$, $\fR_{J_{2i}}$ consist of a singleton set and therefore have log cover size $0$ for any error resolution $\epsilon$. Otherwise, to cover $\fR_{V_{2i - 1}}$ it suffices to bound $\log \cover(\epsilon_{V_{2i - 1}}, \{h \mapsto Wh : \|W^\top - {A^{(i)}}^\top\|_{2, 1} \le a^{(i)}\}, 2t^{(i - 1)})$. Thus, we can apply Lemma~\ref{lem:two_one_norm} to obtain 
	\begin{align*}
	\log \cover(\epsilon_{V_{2i - 1}}, \fR_{V_{2i - 1}}, 2s_{V_{2i - 2}}) \le \tilde{O} \left(\frac{ (a^{(i)}t^{(i - 1)})^2}{\epsilon_{V_{2i - 1}}^2} \right)
	\end{align*}
	Now to cover $D\fR_{V_{2i - 1}}$, it suffices to cover $\{W : \|W - {B^{(i)}}\|_{1, 1} \le b^{(i)}\}$. The $\epsilon$-covering number of a $d_h^2$-dimensional $\ell_1$-ball with radius $b$ w.r.t. $\ell_2$ norm is  $O(\frac{b^2}{\epsilon^2} \log d_h)$. Thus, 
	\begin{align*}
	\log \cover(\epsilon_{J_{2i - 1}}, D\fR_{V_{2i - 1}}, 2s_{V_{2i - 2}}) \le \tilde{O}\left(\frac{(b^{(i)})^2}{\epsilon_{J_{2i - 1}}^2}\right)
	\end{align*}
	Now we define
	\begin{align*}
	\beta^\star &\triangleq \left(\sum_{i} (\laug_{V_{2i - 1}} a^{(i)} t^{(i -1)})^{2/3}  + (\laug_{J_{2i - 1}} b^{(i)})^{2/3}\right)^{3/2} \\&= \left(\sum_{i} (\kappah{i} a^{(i)} t^{(i - 1)})^{2/3}  + (\kappaj{i} b^{(i)})^{2/3}\right)^{3/2}
	\end{align*}
	Now for a fixed error parameter $\epsilon$, we set $\epsilon_O = 0$, $\epsilon_{V_{2i}} = 0$, $\epsilon_{J_{2i}} = 0$ (as the log cover size is 0 anyways), and $\epsilon_{V_{2i -1}} = \epsilon\frac{\laug_{V_{2i -1}}^{-1/3}(a^{(i)}t^{(i - 1)})^{2/3}}{(\beta^\star)^{2/3}}$, $\epsilon_{J_{2i -1}} = \epsilon\frac{\laug_{J_{2i -1}}^{-1/3}(b^{(i)})^{2/3}}{(\beta^\star)^{2/3}}$ Now it follows that $\sum_j \epsilon_{V_{j}} \laug_{V_i} + \epsilon_{J_j} \laug_{J_j} = \epsilon$. Furthermore, under these choices of $\epsilon_{V_i}$, $\epsilon_{J_i}$, we end up with 
	\begin{align*}
	\sum_{i \ge 1} \log \cover(\epsilon_{V_i}, \fR_{V_i}, 2s_{V_{i - 1}}) + \log \cover(\epsilon_{J_i}, D\fR_{V_i}, 2s_{V_{i - 1}}) \\\le \tilde{O}\left(\frac{1}{\epsilon^2} (\beta^\star)^{4/3} \left(\sum_{i} (\kappah{i} a^{(i)} t^{(i - 1)})^{2/3}  + (\kappaj{i} b^{(i)})^{2/3}\right)^{3/2}\right) = \tilde{O}(\epsilon^{-2}(\beta^\star)^2)
	\end{align*}
	Thus, substituting terms into~\eqref{eq:nnrad-1} and collecting sums, we obtain that 
	\begin{align*}
	\log \cover(\epsilon, O_{\augG}, s_{\cI}) \le \tilde{O}( \epsilon^{-2}(\beta^\star)^2)
	\end{align*}
	Now we apply Dudley's entropy theorem to obtain that 
	\begin{align*}
	\rad(\augG) = \tilde{O}\left(\frac{\left(\sum_{i} (\kappah{i} a^{(i)} t^{(i - 1)})^{2/3} + (\kappaj{i}b^{(i)})^{2/3}\right)^{3/2}}{\sqrt{n}}\right)
	\end{align*}
\end{proof}

We now apply~\ref{thm:nnrad} to prove Theorem~\ref{thm:gen_no_union_bound}. 

\begin{proof}[Proof of Theorem~\ref{thm:gen_no_union_bound}]
	We start with Theorem~\ref{thm:nnrad}, which bounds the Rademacher complexity of the augmented loss class $\cL_{\textup{aug}}$. Using $l_{\textup{aug}}(F, x, y)$ to denote the application of this augmented loss on the network $F$, its weights, and data $(x, y)$, we first note that $l_{\textup{0-1}}(F(x), y) \le l_{\gamma}(F(x), y) \le l_{\textup{aug}}(F, x, y)$ for any datapoint $(x, y)$. We used the fact that margin loss upper bounds 0-1 loss, and $l_{\textup{aug}}$ upper bounds margin loss by the construction in Theorem~\ref{thm:sequential_covering}. Thus, applying the standard Rademacher generalization bound, with probability $1 - \delta$ over the training data, it holds that 
	\begin{align}
	\Exp_{(x, y) \sim \cD}[l_{\textup{0-1}}(F(x), y)] &\le \Exp_{(x, y) \sim \cD}[l_{\textup{aug}}(F, x, y)] \\ &\le  \Exp_{(x, y) \sim \cD_n}[l_{\textup{aug}}(F, x, y)] + \rad(\cL_{\textup{aug}}) + \sqrt{\frac{\log(1/\delta)}{n}} \\
	&=  \rad(\cL_{\textup{aug}}) + \sqrt{\frac{\log(1/\delta)}{n}} \tag{by the data-dependent conditions}\notag
	\end{align}
	Plugging in the bound on $\rad(\cL_{\textup{aug}})$ from Theorem~\ref{thm:nnrad} gives the desired result. 
\end{proof}

Finally, to prove Theorems~\ref{thm:gen_union_bound}~and~\ref{thm:geninformal}, we simply take a union bound over the choices of parameters $\sigma_{j' \ot j}, t^{(i)}, a^{(i)}, b^{(i)}$. 
\begin{proof}[Proof of Theorems~\ref{thm:gen_union_bound}~and~\ref{thm:geninformal}]
	We will apply Theorem~\ref{thm:gen_no_union_bound} repeatedly over a grid of parameter choices $t^{(i)}$, $\sigma_{j'\ot j}$, $a^{(i)}$, $b^{(i)}$ (following a technique of~\citet{bartlett2017spectrally}). For a collection $\mathcal{M}$ of nonnegative integers $m_t^{(i)}$, $m_\sigma^{(j' \ot j)}$, $m_a^{(i)}$, $m_b^{(i)}$, $m_{\gamma}$, we apply Theorem~\ref{thm:gen_no_union_bound} choosing $t^{(i)} = \poly(r)^{-1}2^{m_t^{(i)}}$, $\sigma_{j' \ot j} = \poly(r)^{-1}2^{m_{\sigma}^{(j' \ot j)}}$, $a^{(i)} = \poly(r)^{-1}2^{m_a^{(i)}}$, $b^{(i)} = \poly(r)^{-1}2^{m_b^{(i)}}$, $\gamma = 2^{-m_{\gamma}}\poly(r)\max_{i} \sigma_{2r - 1 \ot 2i}$ and using error probability $\delta_{\mathcal{M}} \triangleq \frac{\delta}{2^{\sum_{m \in \mathcal{M}} m + 1}}$. First, we note that by union bound, using the fact that $\sum_{\textup{choices of }\mathcal{M}}\frac{\delta}{2^{\sum_{m \in \mathcal{M}} m + 1}} = \delta$ where $\mathcal{M}$ ranges over nonnegative integers, we get that the generalization bound of Theorem~\ref{thm:gen_no_union_bound}~holds for choices of $\mathcal{M}$ with probability 1 - $\delta$.
	
	Now for the network $F$ at hand, there would have been some choice of $\mathcal{M}$ for which the bound was applied using parameters $\hat{t}^{(i)}$, $\hat{\sigma}_{j' \ot j}$, $\hat{a}^{(i)}$, $\hat{b}^{(i)}$, $\hat{\gamma}$ and 
	\begin{align*}
		\|{W^{(i)}}^\top - {A^{(i)}}^\top\|_{2, 1} &\le \hat{a}^{(i)} = \poly(r)^{-1}2^{m_a^{(i)}} \le \poly(r)^{-1} + 2\|{W^{(i)}}^\top - {A^{(i)}}^\top\|_{2, 1}\\
		\|{W^{(i)}} - {B^{(i)}}\|_{1, 1} &\le \hat{b}^{(i)}= \poly(r)^{-1}2^{m_b^{(i)}} \le \poly(r)^{-1} + 2	\|{W^{(i)}} - {B^{(i)}}\|_{1, 1}\\
		\max_{x \in P_n} \|F_{2i \ot 1}(x)\| &\le \hat{t}^{(i)} = \poly(r)^{-1}2^{m_t^{(i)}} \le \poly(r)^{-1} + 2\max_{x \in P_n} \|F_{2i \ot 1}\|\\
		\max_{x \in P_n} \opnorm{Q_{j' \ot j}(x)} &\le \hat{\sigma}_{j' \ot j} = \poly(r)^{-1}2^{m_{\sigma}^{(j' \ot j)}} \le \poly(r)^{-1} + 2\max_{x \in P_n} \opnorm{Q_{j' \ot j}(x)}
	\end{align*} 
	Furthermore, using $\gamma$ to denote the true margin of the network, we also have $\hat{\gamma} \le \gamma$ and $\frac{\hat{\sigma}_{2r - 1 \ot 2i}}{\hat{\gamma}} \le 4\frac{\max_{x \in P_n} \opnorm{Q_{2r - 1 \ot 2i}(x)}}{\gamma} + \frac{1}{\poly(r)}$. Furthermore, note that the cost we pay in $\sqrt{\frac{\log(1/\delta_{\mathcal{M}})}{n}}$ is $\tilde{O}\left(r \sqrt{\frac{\log(1/\delta)}{n}}\right)$, where $\tilde{O}$ hides polylog factors in $r$ and other parameters. Thus, the bound of Theorem~\ref{thm:gen_union_bound} holds. 
	
	The proof of the simpler Theorem~\ref{thm:geninformal}, follows the same above argument. The only difference is that we union bound over parameters $\sigma, t$ and the matrix norms.
\end{proof}

\begin{proposition}[Dudley's entropy theorem \citep{dudley1967sizes}]
	Let $s = \max_{x \in P_n} \|x\|$ be an upper bound on the largest norm of a datapoint. Then the following bound relates Rademacher complexity to covering numbers: 
	\begin{align*}
		\rad(\cL) \le \inf_{\alpha > 0} \left(\alpha + \int_{\alpha}^{\infty} \sqrt{\frac{\log \cover(\epsilon, \cL, s)}{n}} d\epsilon\right)
		\end{align*}
\end{proposition}
\begin{lemma}
	\label{lem:two_one_norm}
	For reference matrix $A \in \R^{d_1 \times d_2}$, define the class of matrices mapping functions $\cU \triangleq \{h \mapsto Uh : U \in \R^{d_1 \times d_2}, \|U^\top - A^\top\|_{2, 1} \le a\}$. Then 
	\begin{align*}
	\log \cover(\epsilon, \cU, b) \le \frac{2a^2b^2}{\epsilon^2} \log(2d_1d_2)
	\end{align*} 
\end{lemma}
\begin{proof}
	By Lemma 3.2 of~\cite{bartlett2017spectrally}, we can construct cover $\hatcU$ for the class $\{h \mapsto (U - A)h : U \in \R^{d_1 \times d_2}, \|U^\top - A^\top\|_{2, 1} \le a\}$ with the given cover size (Note that in our definition of empirical covering number, the resolution $\epsilon$ is scaled by factor $\frac{1}{n}$ versus theirs). To cover $\cU$ with the same cardinality set, we simply shift all functions in $\hatcU$ by $A$.
\end{proof}

%% file: app_computational_graph.tex
\section{Missing Proofs in Section~\ref{sec:computational_graph}}\label{app:computational_graph}
We first state the proof of Theorem~\ref{thm:covering_graph}.
\begin{proof} [Proof of Theorem~\ref{thm:covering_graph}]
	We prove the theorem by induction on the number of non-input vertices in the vertex set $\cV$. The statement is true if $O$ is the only non-input node in the graph: to cover the graph output with error $\epsilon_O$, we simply cover $\fR_O$. 
	
	Given a family of graphs $\cG$ (with shared edges $\cE$ and nodes $\cV$), we assume the inductive hypothesis that ``for any family of graphs with more than $|\cI|$ input vertices, the theorem statement holds.'' Under this hypothesis, we will show that the theorem statement holds for the graph family $\cG$. 
	
	We take node $V_1$ from the forest ordering $(V_1,\dots, V_m)$ assumed in the theorem. 
	Suppose $V_1$ depends on $C_1,\dots, C_t$, which are assumed to be the input nodes by the definition of forest ordering. 
	We release the node $V_1$ from the graph and obtain a new family $\cut{\cG}{V_1} = \{\cut{G}{V_1}: G\in \cG\}$ with a smaller number of edges than that of $\cG$.

	Define $u(h, x) \triangleq O_{\cut{G}{V_1}}(h,x)$ for $h\in \cD_{V_1}$ and $x\in \cD_\cI$, and $w(x) = V_1(x)$. 
	Then we can check that $u(w(x), x) = O_G(x)$. Let $\cU = \{O_{\cut{G}{V_1}}: G \in \cG\}$, and let $\cW = \fR_{V_1}$. 
	As each function in $\cU$ is $\kappa_{V_1}$-Lipschitz in $V_1$ because of condition 1, and it equals the fixed constant $c$ if $\gnorm{V_1} \ge s_V$ or $\gnorm{C_i} \ge s_{C_i}$, we have $\cU, \cW$ satisfies the conditions of the composition lemma (see Lemma~\ref{lem:composition-shared}). 
	With the lemma, we conclude:
	\begin{align}
	\log \cover(\kappa_{V_1}  \epsilon_{V_1} + \epsilon_u, \cG, s_\cI) \le \log \cover(\epsilon_u, \cU, (s_{V_1}, s_\cI)) + \log \cover(\epsilon_{V_1}, \mathfrak{R}_{V_1}, s_{\pr(V_1)}) \label{eqn:new}
	\end{align}
	
	Note that by the definition of forest ordering, we have that $(V_2,\dots, V_m)$ is a forest ordering of $\cut{G}{V_1}$ and by the assumption 1 of the theorem, we have that $(V_2,\dots, V_m)$ satisfies the condition 1 for the graph family $\cut{\cG}{V_1}$. $\cut{\cG}{V_1}$ has one more input node than $\cG$, so we can invoke the inductive hypothesis on $\cut{\cG}{V_1}$ and obtain
			
	\begin{align}
	\log \cover(\sum_{V\in \cV\backslash (\{V_1, O\}\cup \cI)} \kappa_V\cdot \epsilon_V + \epsilon_O, \cU, (s_{V_1}, s_\cI)) \le   \sum_{V\in \cV\backslash (\{V_1\}\cup \cI)} \log \cover(\epsilon_{V}, \mathfrak{R}_{V}, s_{\pr(V)})\label{eqn:inductive_h}
	\end{align}
	Combining equation~\eqref{eqn:new} and~\eqref{eqn:inductive_h} above, we prove~\eqref{eq:covering_graph-1} for $\cG$, and complete the induction. 
\end{proof}

Below we provide the composition lemma necessary for Theorem~\ref{thm:covering_graph}.
\begin{lemma}\label{lem:composition-shared}
	Suppose 
	\begin{align*}
		\cU \subseteq \{(h,x^{(1)}, \ldots, x^{(m)})\in \cD_h\otimes \cD_x^{(1)} \otimes \cdots \otimes \cD_x^{(m)} \mapsto \cD_u\}
	\end{align*}
	is a family of functions with two arguments and $\cW \subseteq \{x^{(1)}, \ldots, x^{(m)}\in \cD_x^{(1)} \otimes \cdots \otimes \cD_x^{(m)} \mapsto \cD_h\}$ is another family of functions. We overload notation and refer to $x^{(1)}, \ldots, x^{(m)}$ as $x$. The spaces $\cD_h,\cD_x,\cD_u$ all associate with some norms $\gnorm{\cdot}$ (the norms can potentially be different for each space, but we use the same notation for all of them.)
	Assume the following: 
	\begin{itemize}
		\item[1.] All functions in $\cU$ are $\kappa$-Lipschitz in the argument $h$ for any possible choice of $x$: for any $u \in \cU$, $x \in \cD_x$, and $h, h' \in \cD_h$, we have $\gnorm{u(h, x) - u(h', x)} \le \kappa \gnorm{h - h'}$.
		\item[2.] Any function $u\in \cU$ collapses on inputs with large norms: there exists a constant $b$ such that $u(h, x) = b$ if $\gnorm{h} \ge s_h$ or $\gnorm{x^{(i)}} \ge s_x^{(i)}$ for any $i$. 
	\end{itemize}
			
	Then, the family of the composition of $u$ and $w$, $\cZ = \left\{z(x) = u(w(x), x) : u\in \cU, w\in \cW \right\}$,  has covering number bound:
	\begin{align*}
	\log \cover(\kappa\epsilon_w+\epsilon_u, \cZ, s_x) \le \log \cover(\epsilon_w, \cW, s_x) + \log \cover(\epsilon_u,\cU, (s_h,s_x))
	\end{align*}
\end{lemma}
\begin{proof}
		When it is clear from context, we let $\gnorm{x} \le s_x$ denote the statement that $\gnorm{x^{(i)}} \le s_x^{(i)} \ \forall i$. 
	Suppose $P_n$ is a uniform distribution over $n$ data points $\{x_1,\dots, x_n\}\subset \cD_x$ with norms not larger than $s_x$.  
	Given function $u\in \cU$ and $w\in \cW$, we will construct a pair of functions such that $\hatu(\hatw(x), x)$ covers $u(w(x), x)$. 
		We will count (in a straightforward way) how many distinct pairs of functions we have construct for all the $(u,w)$ pairs at the end of the proof. 
	
	Let $P'$ be the uniform distribution over $\{x_i : \gnorm{x_i} \le s_x\}$, and suppose $\hat{\cW}$ is a $\epsilon_w \sqrt{\frac{n}{|\supp(P')|}}$  error cover of $\cW$ with respect to the metric $L_2(P', \gnorm{\cdot})$. We note that $\hat{\cW}$ has size at most $\cover(\epsilon_w, \cW, s_x)$. We found $\hat{w} \in \cW$ such that $\hat{w}$ is $\epsilon_w$-close to $w$ in metric $L_2(P', \gnorm{\cdot})$. 
	Let $\hat{h}_i$ denote $\hat{w}(x_i)$. Let $Q'$ be the uniform distribution over $\{(\hath_i, x_i) : \gnorm{\hath_i} \le s_h, \gnorm{x_i} \le s_x\}$, and let $Q$ be the uniform distribution over all $n$ points, $\{(\hath_1, x_1), \ldots, (\hath_n, x_n)\}$.  
	Now we construct a intermediate cover $\hatcU'$ (that depends on $\hat{w}$ implicitly) that covers $\cU$ with $\epsilon_u \sqrt{\frac{n}{|\supp(Q')|}}$ error with respect to the metric $L_2(Q', \gnorm{\cdot})$. 
	We augment this to a cover $\hatcU$ that covers $\cU$ with respect to metric $L_2(Q, \gnorm{\cdot})$ as follows: for every $\hatu' \in \hatcU'$, add the function $\hatu$ to $\hatcU$ with
	\begin{align*}
		\hatu(h, x) = \begin{cases}
		\hatu'(h, x) & \textup{ if } \gnorm{h} \le s_h, \gnorm{x} \le s_x\\
		b & \textup{ otherwise }
		\end{cases}
	\end{align*}
	Note that by construction, the size of $\hatcU$ is at most $\cover(\epsilon_u, \cU, (s_h, s_x))$. Now let $\hatu' \in \hatcU'$ be the cover element for $u $ w.r.t. $L_2(Q, \gnorm{\cdot})$, and $\hatu$ be the corresponding cover element in $\hatcU$. Because $\hatu(\hath, x) = b = u(\hath, x)$ when $\gnorm{\hath} \ge s_h$ or $\gnorm{x^{(i)}} \ge s_x^{(i)}$ for some $i$,
	\begin{align}
	\Exp_{\hath, x \sim Q}\left[\gnorm{\hat{u}(\hath, x )  - u(\hath, x)}^2\right] = \frac{|\supp(Q')|}{n}\Exp_{\hath, x \sim Q'} \left[\gnorm{\hat{u}'(\hath, x )  - u(\hath, x)}^2\right] \le \epsilon_u^2 \label{eqn:3}
	\end{align}
Then we bound the difference between $u(\hath, x)$ and $u(h, x)$ by Lipschitzness; since $u(\hath, x) = u(h, x) = b$ when $\gnorm{x} > s_x$,
	\begin{align}
\Exp_{\hath, x\sim Q} \left[\gnorm{u(\hath, x) - u(h,x)}^2\right] \le \kappa^2 \frac{|\supp(P')|}{n}\Exp_{\hath, x\sim  P'} \left[\gnorm{\hath - h}^2\right] \le \kappa^2 \epsilon_w^2\label{eqn:4}
	\end{align}
	where in the last step we used the property of the cover $\cG$. 
 Finally, by triangle inequality, we get that 
	\begin{align}
	&	\|\hat{u}(\hatw(x), x )-u(w(x),x) \|_{L_2(P_n,\gnorm{\cdot})} \notag\\
	& \le \|\hat{u}(\hatw(x), x )  - u(\hatw(x), x) \|_{L_2(P_n,\gnorm{\cdot})}+	\|u(\hatw(x), x) - u(w(x),x) \|_{L_2(P_n,\gnorm{\cdot})}\notag \\
	& \le \kappa\epsilon_w + \epsilon_u \tag{by equation~\eqref{eqn:3} and~\eqref{eqn:4} and definition of $h_i,\hath_i$}
	\end{align}
	Finally we count how many $(\hat{w},\hat{u})$ we have constructed: $\hat{\cW}$ is of size at most $\cover(\epsilon_w, \cW, s_x)$. and for every $\hatw \in \hat{\cW}$, we've constructed a family of functions $\hatcU$ (that depends on $\hatw$) of size at most $\cover(\epsilon_u,\cU, (s_h, s_x))$. Therefore, the total size of the cover is at most $\cover(\epsilon_w, \cW, s_x)\cdot \cover(\epsilon_u,\cU, (s_h,s_x))$. 
\end{proof}

%% file: app_augmentation.tex
\section{Missing Proofs in Section~\ref{sec:lipschitzaug}}\label{app:augmentation}
We first state the proofs of Theorem~\ref{thm:lip}~and Theorem~\ref{thm:sequential_covering}, which follow straightforwardly from the technical tools developed in Section~\ref{app:technical_aug}.
\begin{proof}[Proof of Theorem~\ref{thm:lip}]
									Fix any forest ordering $\cS$ of $\augG$. Fix $\tildeG \in \augG$. Let $\cS'$ be the prefix sequence of $\cS$ ending in $V_i$. Note that $\cS'$ will not contain any $J_j$ or $V_j$ for $j > i$, as $V_j$ and $J_j$ will still depend on a non-input node (namely, $V_{j - 1}$). Thus, we can fit $\cut{\tildeG}{\cS'}$ under the framework of Lemma~\ref{lem:lip}, where we set $k = \depthlocal - i$ and identify $f_j$ with $R_{V_{i + j}}$. We set $m = i$, and identify $A_{m'}$ with $J_i \cdots J_{m'}$ (where $J_{j}$ may depend on input variables or itself be an input variable for $1 \le j \le i$, but this does not matter for our purposes). Then that to apply Lemma~\ref{lem:lip}, we set $\tau_{j' \ot i'} = \kappa_{j' + i \ot i' + i}$, $\tau_{j' \ot 1, m'} = \kappa_{j' + i \ot m'}$, and $\bar{\tau}_{j} = \ulip_{i + j}$. Now we can apply Lemma~\ref{lem:lip} to conclude that $\cut{\tildeG}{\cS'}$ is $\laug_{V_i}$-Lipschitz in $V_i$ for any $1 \le i \le \depthlocal$.
	
	Now we prove release-Lipschitzness for a prefix sequence $\cS'$ of $\cS$ that ends in node $J_i$. For all $j \ne i$, fix $D_j \in \cD_{J_j}$. It suffices to show that the function $Q$ defined by 
	\begin{align*}
	Q(J_i) \triangleq & \prod_{j \le i \le j'} \one[\le \kappa_{j' \ot j}] (\opnorm{D_{j'} \cdots D_{i + 1} J_i D_{i - 1} \cdots D_j})  \\
	&\times \prod_{j' \ge i + 1} \one[\le \kappa_{j' \ot i + 1}](\opnorm{D_{j'} \cdots D_{i + 1}})\times \prod_{j \le i - 1}\one[\le \kappa_{i - 1 \ot j}](\opnorm{D_{i - 1} \cdots D_{j}})
	\end{align*}
	is $\laug_{J_i}$-Lipschitz in the value of $J_i$. This is because after fixing all other inputs besides $J_i$, we can write $O_{\cut{\tildeG}{\cS'}}$ in the form $Q(J_i) a + 1$, where $a$ may depend on the other inputs but not $J_i$ and $|a| \le 1$. Now we simply apply Lemma~\ref{lem:QJ_i_lipschitz} to conclude that $Q(J_i)$, and therefore $O_{\cut{\tildeG}{\cS'}}$, is $\laug_{J_i}$-Lipschitz. 
	\end{proof}

\begin{proof}[Proof of Theorem~\ref{thm:sequential_covering}]
		We first construct an augmented family of graphs $\cG'$ sharing the same vertices and edges as $\cG$. For $G \in \cG$, we add $G'$ to $\cG'$ computing
	\begin{align*}
	O_{G'}(x) = (O_G(x) - 1)\prod_{i = 1}^\depthlocal \one[\le s_{V_i}] (\|V_i(x)\|) + 1
	\end{align*}
	
	This is achieved by modifying the family of output rules as follows: 
	\begin{align*}
	R'_O(x, v_1, \ldots, v_\depthlocal) = (R_{O}(x, v_1, \ldots, v_\depthlocal) - 1)\prod_{i = 1}^\depthlocal \one[\le s_{V_i}](\|v_i\|) + 1
	\end{align*}
	where $x \in \cD_{\cI}$ and $v_i \in \cD_{V_i}$. 	
	We can also apply Claim~\ref{claim:indicator_stack} to conclude that $R'_O$ outputs values in $[0, 1]$. Furthermore, as the function $\one[\le s_{V_i}](\|v_i\|)$ is $s_{V_i}^{-1}$-Lipschitz in $v_i$, by the product property for Lipschitzness, $R'_O$ is $(c_i + s_{V_i})^{-1}$-Lipschitz in $v_i$. Now we apply Theorem~\ref{thm:lip} to obtain a graph family $\augG$ that is $\{\tilde{\kappa}_V\}$-release-Lipschitz with respect to any forest ordering on $(\augV, \augE)$ for parameters $\{\tilde{\kappa}_V\}$ defined in the theorem statement. Furthermore, by the construction of our augmentation and application of Claim~\ref{claim:indicator_stack}, it follows that 
	\begin{align*}
	&\tilde{R}_O(x, v_1, \ldots, v_\depthlocal, D_1, \ldots, D_\depthlocal) =\\ &(R_{O}(x, v_1, \ldots, v_\depthlocal) - 1)\prod_{i = 1}^\depthlocal \one[\le s_{V_i}] (\|v_i\|) \prod_{1 \le i \le j \le \depthlocal} \one[\le \kappa_{j \ot i}] (\opnorm{D_j \cdots D_i})+ 1
	\end{align*}
	and in particular outputs the constant value $1$ when $\|v_i\| > 2s_{V_i}$ or $\|D_i\| > 2\kappa_{i \ot i}$. As this is a property of the output rule $\tilde{R}_O$ itself, it is clear that condition 2 of Theorem~\ref{thm:lip} holds for any forest ordering on $(\augV, \augE)$. Now we can apply Theorem~\ref{thm:lip}: 	\begin{align*}
	\log \cover(\sum_{i \ge 1} (\laug_{V_i} + \laug_{J_i})\epsilon_V + \epsilon_O, O_{\augG}, s_{\cI}) \le \sum_{i \ge 1} \log \cover(\epsilon_{V_i}, \fR_{V_i}, 2s_{V_{i - 1}}) \\+ \log \cover(\epsilon_{J_i}, D\fR_{V_i}, 2s_{V_{i - 1}}) + \log \cover(\epsilon_O, \augR_O, \{2s_{V_i}\} \cup \{I\} \cup \{2s_{J_i}\}_{i \ge 1})
	\end{align*}
	Now all terms match~\eqref{eq:sequential_covering-bound} except for the term $\log \cover(\epsilon_O, \augR_O, \{2s_{V_i}\} \cup \{I\} \cup \{2s_{J_i}\}_{i \ge 1})$. First, we note that all functions in $\augR_O$ can be written in the form 
	\begin{align*}
	\tilde{R}_O(x, v_1, \ldots, v_q, D_1, \ldots,  D_q) = (R_O(x, v_1, \ldots, v_q) - 1) Q(v_1, \ldots, v_q, D_1, \ldots, D_q) + 1
	\end{align*}
	where the function $Q$ is the same for all $\tilde{R}_O \in \augR_O$. It follows that to cover $\augR_O$, we can first obtain a cover $\hat{\fR}_O$ of $\fR_O$ and then apply the operation $\hat{r} \mapsto (\hat{r} - 1)Q + 1$ to each element in $\hat{\fR}_O$. Thus, we get the equivalence
	\begin{align*}
	\log \cover(\epsilon_O, \augR_O, \{2s_{V_i}\}_{i \ge 0}\cup \{2s_{J_i}\}_{i \ge 1}) = \log \cover (\epsilon_O, \fR_O, \{2s_{V_i}\} \cup \{I\})
	\end{align*}
	This allows us to conclude~\eqref{eq:sequential_covering-bound}. 
	Finally, we note that as the augmentation operations are in the form of those considered in Claim~\ref{claim:indicator_stack}, it follows that $O_{\tildeG}$ upper  bounds $O_{G}$.
\end{proof}

%% file: low_level_lipschitzfy.tex
\section{Technical Tools for Lipschitz Augmentation}\label{app:technical_aug}
In this section, we develop the technical tools needed for proving Theorem~\ref{thm:lip}. The main result in this section is our Lemma~\ref{lem:lip}, which essentially states that augmenting the loss with a product of Jacobians (plus additional matrices meant to model previous Jacobian nodes already released from the computational graph) will make the loss Lipschitz.

For this section, we say a function $J$ taking input $x \in \cD$ and outputting an operator mapping $\cD$ to $\cD'$ is $\kappa$-Lipschitz if $\opnorm{J(x) - J(x')} \le \kappa\|x - x'\|$ for any $x, x'$ in its input domain. We will consider functions $f_1, \ldots, f_k$, where $f_i : \cD_{i - 1} \rightarrow \cD_i$ and $\cD_0$ is a compact subset of some normed space. For ease of notation, we use $\|\cdot \|$ to denote the (possibly distinct) norms on $\cD_{0}, \ldots, \cD_{k}$. For $1 \le i \le j \le k$, Let $f_{j \ot i} : \cD_{i - 1} \rightarrow \cD_j$ denote the composition 
\begin{align*}
	f_{j \ot i} \triangleq f_{j} \circ \cdots \circ f_i
\end{align*}
For convenience in indexing, for $(i, j)$ with $i > j$, we will set $f_{j \ot i} : \cD_{i - 1} \rightarrow \cD_{i - 1}$ to be the identity function. 

Finally consider a real-valued function $g :  \cD_0 \otimes \cdots \otimes \cD_k \rightarrow [0, 1]$ and define the composition $z : \cD_0 \mapsto [0, 1]$ by
\begin{align*}
z(x) = g(x, f_{1 \ot 1}(x), \ldots, f_{k \ot 1}(x))
\end{align*}
We will construct a ``Lipschitz-fication'' for the function $z$. 

Let $A_1, \ldots, A_m$ denote a collection of linear operators that map to the space $\cD_0$. We will furthermore use $J_{j \ot i, m'}$ to denote the $i$-to-$j$ Jacobian, i.e. 
\begin{align*}
J_{j \ot i, m'} \triangleq Df_{j \ot i} \circ f_{i - 1 \ot 1}
\end{align*}
When $i = 1$ and $0 \le j \le k$, we will also consider products between $1$-to-$j$ Jacobians and the matrices $A_{m'}$: define 
\begin{align*} 
J_{j \ot 1, m'} \triangleq (Df_{j \ot 1}) A_{m'}
\end{align*}
Note in particular that $J_{0 \ot 1, m'} = A_{m'}$. 
\begin{lemma} \label{lem:lip} [Lipschitz-fication]
	Following the notation in this section, suppose that $g$ is $c_{k'}$-Lipschitz in its $(k' + 1)$-th argument for $0 \le k' \le k$. Suppose that $Df_{j \ot j}$ is $\bar{\tau}_j$-Lipschitz for all $1 \le j \le k$. For any $(i, j)$ with $1 \le i \le j \le k$, let $\tau_{j \ot i}$ be parameters that intend to be a tight bound on $\opnorm{J_{j \ot i}}$, and also define $\tau_{j \ot 1, m'}$ which will bound $\opnorm{J_{j \ot 1, m'}}$. Define the augmented function $\bar{z}: \cD_0 \mapsto [0, 1]$ by
	\begin{align*}
		\tilde{z}(x) = (z(x) - 1) \prod_{2\le i \le j}  \one[\le \tau_{j \ot i}](\opnorm{J_{j \ot i}(x)}) \prod_{0 \le j \le k, m'} \one[\le \tau_{j \ot 1, m'}](\opnorm{J_{j \ot 1, m'}}) + 1
	\end{align*} 
	Define $\tau^\star$, a Lipschitz parameter for $\tilde{z}$, by 
	\begin{align*}
		\tau^\star \triangleq & \sum_{0 \le j \le k} 3c_j\tau_{j \ot 1}  \\&+ 18 \sum_{1 \le i \le j \le k} \frac{\sum_{i' = i}^j \bar{\tau}_{i'} \tau_{j \ot i' + 1} \tau_{i' - 1 \ot 1} \tau_{i' -1 \ot i}}{\tau_{j \ot i}} \\&+ 18 \sum_{1  \le j \le k, m'} \frac{\sum_{i' = 1}^j \bar{\tau}_{i'} \tau_{j \ot i' + 1} \tau_{i' - 1 \ot 1} \tau_{i' -1 \ot 1, m'}}{\tau_{j \ot 1, m'}}
	\end{align*}
	where for convenience we let $\tau_{j \ot i} = 1$ when $j < i$. 
	Then $\tilde{z}$ is $\tau^\star$-Lipschitz in $x$. 
\end{lemma}
\newcommand{\cfunc}{\mathcal{Q}}
\newcommand{\pfunc}{\succ_{\cfunc}}
\begin{proof}
	For ease of notation, we will first define for any $(i, j)$ with $1 \le i \le j \le k$, $Q_{j\ot i} \triangleq \one[\le \tau_{j \ot i}](\opnorm{J_{j \ot i}})$ and for $(j, m')$ with $0 \le j \le k$, $Q_{j \ot 1, m'} \triangleq \one[\le \tau_{j \ot 1, m'}](\opnorm{J_{j \ot 1, m'}})$. Note in particular that $Q_{0 \ot 1, m'}$ is always a constant function. We will also let $\cfunc$ denote the collection of functions 
	\begin{align*}
	\mathcal{Q} = \{Q_{i \ot j}\}_{1 \le i \le j \le k} \cup \{Q_{j \ot 1, m'}\}_{0 \le j \le k, 1 \le m' \le m}
	\end{align*}
	We define the following order $\pfunc $ on this collection of functions: 
	\begin{align*}
	Q_{0 \ot 1, m} \pfunc \cdots \pfunc Q_{0 \ot 1, 1} \\
	\pfunc Q_{1 \ot 1} \pfunc Q_{1 \ot 1, m} \pfunc \cdots \pfunc Q_{1 \ot 1, 1}\\
	\pfunc Q_{2 \ot 2} \pfunc Q_{2 \ot 1} \pfunc Q_{2 \ot 1, m} \pfunc \cdots \pfunc Q_{2 \ot 1, 1}\\
	\vdots \\
	\pfunc Q_{k \ot k} \pfunc \cdots \pfunc Q_{k \ot 1} \pfunc Q_{k \ot 1, m'} \pfunc \cdots \pfunc Q_{k \ot 1, 1}
	\end{align*}
					 We will first show that $\exists C > 0$ such that $\forall x$ and $\forall \nu$ with $\|\nu\| < C$, $|\tilde{z}(x) - \tilde{z}(x + \nu)| \le \tau^\star \|\nu\|$. To use this statement to conclude that $\tilde{z}$ is $\tau^\star$-Lipschitz, we note that for arbitrary $x$ and $\nu$, we can divide the segment between $x$ and $x + \nu$ into segments of length at most $C$, and apply the above statement on each segment. First, define for $0 \le j \le k$
	\begin{align*}
	\gamma_j(x, \nu) \triangleq &g(f_{0 \ot 1}(x), \ldots, f_{j - 1 \ot 1}(x), f_{j \ot 1}(x), f_{j + 1 \ot 1}(x + \nu), \ldots, f_{k \ot 1}(x + \nu)) \\ -& g(f_{0 \ot 1}(x), \ldots, f_{j - 1 \ot 1}(x), f_{j \ot 1}(x + \nu), f_{j + 1 \ot 1}(x + \nu), \ldots, f_{k \ot 1}(x + \nu))
	\end{align*}
	Next, define the telescoping differences
	\begin{align}
	\delta_j(x, \nu) \triangleq &\gamma_j(x, \nu) \prod_{Q \succeq_{\cfunc} Q_{j \ot 1, 1}} Q(x) \prod_{Q_{j \ot 1, 1} \pfunc Q} Q(x + \nu) \ \forall 0 \le j \le k\label{eq:delta_def}\\
	\Delta_{j \ot i}(x, \nu) \triangleq &(Q_{j \ot i}(x) - Q_{j \ot i}(x + \nu)) \prod_{Q \pfunc Q_{j \ot i}} Q(x) \prod_{Q_{j \ot i} \pfunc Q} Q(x + \nu) \ \forall 1 \le i \le j \le k\label{eq:Delta_def} \\
	\begin{split}
	\Delta_{j \ot 1, m'}(x, \nu) \triangleq & (Q_{j \ot 1, m'}(x) - Q_{j \ot 1, m'}(x + \nu))\cdot \\ &\prod_{Q \pfunc Q_{j \ot 1, m'}} Q(x) \prod_{Q_{j \ot 1, m'} \pfunc Q} Q(x + \nu) \ \forall 0 \le j \le k \label{eq:Delta_m_def}
	\end{split}
	\end{align}
	Now note that by Claim~\ref{claim:diff_expansion}, we have the bound
	\begin{align*}
	|\tilde{z}(x) - \tilde{z}(x + \nu)| &\le \sum_{0 \le j \le k} |\delta_j(x, \nu)| + \sum_{1 \le i \le j \le k} |\Delta_{j \ot i}(x, \nu)| + \sum_{0 \le j \le k, m'} |\Delta_{j \ot 1, m'}(x, \nu)|
	\end{align*}
	Define $\bar{\tau}$ to be the Lipschitz constant of $J_{j \ot i}$ on $\cD_0$ for all $1 \le i \le j \le k$ guaranteed by Claim~\ref{claim:all_lipschitz}. First, note that $\Delta_{0 \ot 1, m'} = 0$ for all $m'$. Thus, by Claims~\ref{claim:Delta_err} and~\ref{claim:delta_err}, it follows that 
	\begin{align}
	\begin{split}
	&|\tilde{z}(x) - \tilde{z}(x + \nu)| \le  \sum_{0 \le j \le k} c_j (2 \tau_{j \to 1} + \frac{\bar{\tau}}{2} \|\nu\|)\|\nu\| \\
	& + \sum_{1 \le i \le j \le k} \|\nu\|\frac{\sum_{i' = i}^{j} (2\tau_{j \ot i' + 1} + \bar{\tau} \|\nu \|)\bar{\tau}_{i'} (2\tau_{i' - 1 \ot 1} + \frac{\bar{\tau}}{2} \|\nu\|)2\tau_{i' - 1 \ot i}}{\tau_{j \ot i}} \\ 
	& + \sum_{1 \le j \le k, 1 \le m' \le m} \|\nu\|\frac{\sum_{i' = 1}^{j} (2\tau_{j \ot i' + 1} + \bar{\tau} \|\nu \|)\bar{\tau}_{i'} (2\tau_{i' - 1 \ot 1} + \frac{\bar{\tau}}{2} \|\nu\|)2\tau_{i' - 1 \ot 1, m'}}{\tau_{j \ot 1, m'}} \label{eq:lip-1} 
	\end{split}
	\end{align}
	Now note that if $\|\nu \| \le \frac{2\min_{i \le j} \tau_{j \ot i}}{\bar{\tau}}$, then it follows that $2\tau_{j \ot i} + \frac{\bar{\tau}}{2}\|\nu\| \le 3\tau_{j \ot i} \forall i \le j$. Substituting into~\eqref{eq:lip-1}, we get that $\forall x, \|\nu \| \le \frac{2\min_{i \le j} \tau_{j \ot i}}{\bar{\tau}}$, 
	\begin{align*}
	|\tilde{z}(x) - \tilde{z}(x + \nu)| \le & \|\nu\|\sum_{0 \le j \le k} 3c_j\tau_{j \ot 1}  \\&+ \|\nu\| 18 \sum_{1 \le i \le j \le k} \frac{\sum_{i' = i}^j \bar{\tau}_{i'} \tau_{j \ot i' + 1} \tau_{i' - 1 \ot 1} \tau_{i' -1 \ot i}}{\tau_{j \ot i}} \\&+ \|\nu\|18 \sum_{1  \le j \le k, m'} \frac{\sum_{i' = 1}^j \bar{\tau}_{i'} \tau_{j \ot i' + 1} \tau_{i' - 1 \ot 1} \tau_{i' -1 \ot 1, m'}}{\tau_{j \ot 1, m'}}  \\= & \tau^\star \|\nu\|
	\end{align*}
	It follows that $\tilde{z}$ is $\tau^\star$-Lipschitz. \end{proof}

\begin{claim}
	\label{claim:jacobian_diff}
	In the setting of Lemma~\ref{lem:lip}, for $1 \le i \le j \le k$, we can expand the error $J_{j \ot i}(x) - J_{j \ot i}(x + \nu)$ as follows:
	\begin{align}
	J_{j \ot i}(x) - J_{j \ot i}(x + \nu) = \sum_{i' = i}^j J_{j \ot i' + 1}(x + \nu) (J_{i' \ot i'}(x) - J_{i' \ot i'}(x + \nu)) J_{i' - 1 \ot i}(x) \label{eq:jacobian_diff-1}
	\end{align}
	Furthermore, for $1 \le j \le k, m'$, we can expand the error $J_{j \ot 1, m'}(x) - J_{j \ot 1, m'}(x + \nu)$ as follows: 
	\begin{align}
	J_{j \ot 1, m'}(x) - J_{j \ot 1, m'}(x + \nu) = \sum_{i' = 1}^j J_{j \ot i' + 1}(x + \nu) (J_{i' \ot i'}(x) - J_{i' \ot i'}(x + \nu)) J_{i' - 1 \ot 1, m'}(x) \label{eq:jacobian_diff-2}
	\end{align}
\end{claim}

\begin{proof}
	We will first show~\eqref{eq:jacobian_diff-1} by inducting on $j - i$. The base case $j = i$ follows by definition, as we can reduce $J_{i \ot i + 1}$ and $J_{i - 1 \ot i}$ to constant-valued functions that output the identity matrix.
	
	For the inductive step, we use Claim~\ref{claim:chain_rule} to expand
	\begin{align}
	J_{j \ot i}(x) - J_{j \ot i}(x + \nu) =& J_{j \ot i + 1}(x) J_{i \ot i}(x) - J_{j \ot i + 1}(x + \nu) J_{i \ot i}(x + \nu)   \nonumber \\
	=& (J_{j \ot i + 1}(x) - J_{j \ot i + 1}(x + \nu))J_{i \ot i}(x) \nonumber \\
	& + J_{j \ot i + 1}(x + \nu) (J_{i \ot i}(x) - J_{i \ot i}(x + \nu)) \nonumber \\
	= & \sum_{i' = i+1}^j J_{j \ot i' + 1}(x + \nu) (J_{i' \ot i'}(x) - J_{i' \ot i'}(x + \nu)) J_{i' - 1 \ot i + 1}(x) J_{i \ot i}(x) \tag{by the inductive hypothesis}\\
	& + J_{j \ot i + 1}(x + \nu) (J_{i \ot i}(x) - J_{i \ot i}(x + \nu)) \nonumber \\
	= & \sum_{i' = i+1}^j J_{j \ot i' + 1}(x + \nu) (J_{i' \ot i'}(x) - J_{i' \ot i'}(x + \nu)) J_{i' - 1 \ot i}(x)  \tag{by Claim~\ref{claim:chain_rule}}\\
	& + J_{j \ot i + 1}(x + \nu) (J_{i \ot i}(x) - J_{i \ot i}(x + \nu)) \nonumber \\
	= &\sum_{i' = i}^j J_{j \ot i' + 1}(x + \nu) (J_{i' \ot i'}(x) - J_{i' \ot i'}(x + \nu)) J_{i' - 1 \ot i}(x) \nonumber
	\end{align}
	as  desired.
	
	To prove~\eqref{eq:jacobian_diff-2}, we first note that by definition, $J_{j \ot 1, m'}(x) = J_{j \ot 1}(x) J_{0 \ot 1, m'}$, so 
	\begin{align}
	&J_{j \ot 1, m'}(x) - J_{j \ot 1, m'}(x + \nu)\\&= (J_{j \ot 1}(x) - J_{j \ot 1}(x + \nu))J_{0 \ot 1, m'} \notag \\
	&= \sum_{i' = 1}^j J_{j \ot i' + 1}(x + \nu) (J_{i' \ot i'}(x) - J_{i' \ot i'}(x + \nu)) J_{i' - 1 \ot 1}(x) J_{0 \ot 1, m'} \tag{by~\eqref{eq:jacobian_diff-1}}\\
	&= \sum_{i' = 1}^j J_{j \ot i' + 1}(x + \nu) (J_{i' \ot i'}(x) - J_{i' \ot i'}(x + \nu)) J_{i' - 1 \ot 1, m'}(x) \tag{since $J_{i' - 1 \ot 1}(x) J_{0 \ot 1, m'} = J_{i' -1 \ot 1, m'}(x)$}
	\end{align}
\end{proof}

\begin{claim}
	\label{claim:jacobian_err}
	In the setting of Lemma~\ref{lem:lip}, suppose that $J_{j \ot i}$ is $\bar{\tau}$-Lipschitz for all $1 \le i \le j \le k$. Then we can bound the operator norm error in the Jacobian by
	\begin{align}
	\begin{split}
	\opnorm{J_{j \ot i}(x) - J_{j \ot i}(x + \nu)} \le \\ \|\nu\|\sum_{i' = i}^{j} (\opnorm{J_{j \ot i' + 1}(x)} + \bar{\tau} \|\nu \|)\bar{\tau}_{i'} (\opnorm{J_{i' - 1 \ot 1}(x)} + \frac{\bar{\tau}}{2} \|\nu\|)\opnorm{J_{i' - 1 \ot i}(x)} 
	\label{eq:jacobian_err-1}
	\end{split}
	\end{align}
	Likewise, we can bound the operator norm error in the product between Jacobian and auxiliary matrices by
	\begin{align}
	\begin{split}
	\label{eq:jacobian_err-2}
	\opnorm{J_{j \ot 1, m'}(x) - J_{j \ot 1, m'}(x + \nu)} \le \\ \|\nu\|\sum_{i' = 1}^{j} (\opnorm{J_{j \ot i' + 1}(x)} + \bar{\tau} \|\nu \|)\bar{\tau}_{i'} (\opnorm{J_{i' - 1 \ot 1}(x)} + \frac{\bar{\tau}}{2} \|\nu\|)\opnorm{J_{i' - 1 \ot 1, m'}(x)} 
	\end{split}
	\end{align}
\end{claim}

\begin{proof}
	We will first prove~\eqref{eq:jacobian_err-1}, as the proof of~\eqref{eq:jacobian_err-2} is nearly identical. Starting from~\eqref{eq:jacobian_diff-1} of Claim~\ref{claim:jacobian_diff}, we have
	\begin{align*}
	J_{j \ot i}(x) - J_{j \ot i}(x + \nu) = \sum_{i' = i}^j J_{j \ot i' + 1}(x + \nu) (J_{i' \ot i'}(x) - J_{i' \ot i'}(x + \nu)) J_{i' - 1 \ot i}(x)
	\end{align*}
	By triangle inequality and the fact that $J_{j' \ot i'}$ is $\bar{\tau}$-Lipschitz $\forall i' \le j'$, it follows that 
	\begin{align}
	&\opnorm{J_{j \ot i}(x) - J_{j \ot i}(x + \nu)} \\&\le \sum_{i' = i}^{j} \opnorm{J_{j \ot i' + 1}(x + \nu)} \opnorm{J_{i' \ot i'}(x) - J_{i'\ot i'}(x + \nu)} \opnorm{J_{i' - 1 \ot i}(x)} \nonumber \\
	&\le \sum_{i' = i}^{j} (\opnorm{J_{j \ot i' + 1}(x)} + \bar{\tau} \|\nu\|)\opnorm{J_{i' \ot i'}(x) - J_{i'\ot i'}(x + \nu)} \opnorm{J_{i' - 1 \ot i}(x)} \label{eq:specsumbound}
	\end{align}
	Next, we note that 
	\begin{align}
	\opnorm{J_{i' \ot i'}(x) - J_{i'\ot i'}(x + \nu)} &= \opnorm{Df_{i' \ot i'}[f_{i' - 1 \ot 1}(x)] - Df_{i'  \ot i'}[f_{i' - 1 \ot 1}(x + \nu)]} \nonumber \\
	&\le \bar{\tau}_{i'} \|f_{i' - 1 \ot 1}(x) - f_{i' - 1 \ot 1}(x + \nu)\| \nonumber \\
	&\le \bar{\tau}_{i'} (\opnorm{J_{i' - 1 \ot 1}(x)} + \frac{\bar{\tau}}{2} \|\nu\|) \|\nu\| \tag{applying Claim~\ref{claim:finite_change}}
	\end{align}
	Plugging the above into \eqref{eq:specsumbound}, we get $\eqref{eq:jacobian_err-1}$. To prove~\eqref{eq:jacobian_err-2}, we start from~\eqref{eq:jacobian_diff-2} and follow the same steps as above.
\end{proof}
\begin{claim}
	\label{claim:Delta_err}
	In the setting of Lemma~\ref{lem:lip}, suppose that $J_{j \ot i}$ is $\bar{\tau}$-Lipschitz for all $1 \le i \le j \le k$. Then we can upper bound the error terms corresponding to the indicators by
	\begin{align}
	|\Delta_{j \ot i}(x, \nu)| \le \|\nu\|\frac{\sum_{i' = i}^{j} (2\tau_{j \ot i' + 1} + \bar{\tau} \|\nu \|)\bar{\tau}_{i'} (2\tau_{i' - 1 \ot 1} + \frac{\bar{\tau}}{2} \|\nu\|)2\tau_{i' - 1 \ot i}}{\tau_{j \ot i}}
	\label{eq:Delta_err-i-j}
	\end{align}
	Likewise, the following upper bound holds for all $(j, m')$ with $1 \le j \le k, 1 \le m' \le m$: 
	\begin{align}
	|\Delta_{j \ot 1, m'}(x, \nu)| \le \|\nu\|\frac{\sum_{i' = 1}^{j} (2\tau_{j \ot i' + 1} + \bar{\tau} \|\nu \|)\bar{\tau}_{i'} (2\tau_{i' - 1 \ot 1} + \frac{\bar{\tau}}{2} \|\nu\|)2\tau_{i' - 1 \ot 1, m'}}{\tau_{j \ot 1, m'}}
	\label{eq:Delta_err-j-m}
	\end{align}
\end{claim}
\begin{proof}
	We will prove~\eqref{eq:Delta_err-i-j} as the proof of~\eqref{eq:Delta_err-j-m} is analogous. Note that as $\one[\le \tau_{j \ot i}]$ is $\frac{1}{\tau_{j  \ot i}}$-Lipschitz in its argument, we have 
	\begin{align}
	|Q_{j \ot i}(x) - Q_{j \ot i}(x + \nu)| &= |\one[\le \tau_{j \ot i}](\opnorm{J_{j \ot i}(x)}) - \one[\le \tau_{j \ot i}](\opnorm{J_{j \ot i}(x + \nu)})| \notag \\
	&\le \frac{1}{\tau_{j  \ot i}} |\opnorm{J_{j \ot i}(x)} - \opnorm{J_{j \ot i}(x + \nu)}| \notag\\
	&\le \frac{1}{\tau_{j  \ot i}} \opnorm{J_{j \ot i}(x) - J_{j \ot i}(x + \nu)} \notag
	\end{align}
	Plugging this into our definition for $\Delta_{j \ot i}$~\eqref{eq:Delta_def}, it follows that 
	\begin{align}
|\Delta_{j \ot i}(x, \nu)| \le \frac{1}{\tau_{j  \ot i}} \opnorm{J_{j \ot i}(x) - J_{j \ot i}(x + \nu)} \prod_{Q \pfunc Q_{j \ot i}} Q(x) \prod_{Q_{j \ot i} \pfunc Q} Q(x + \nu) \label{eq:Delta_err-1}
	\end{align}
	Now we define the set $\mathcal{E}$ by 
	\begin{align*}
	\mathcal{E} = \cap_{i \le i' \le j} \{&x : \opnorm{J_{j \ot i' + 1}(x)} \le 2\tau_{j \ot i' + 1}, \opnorm{J_{i' - 1 \ot 1}(x)} \le 2 \tau_{i' - 1 \ot 1},\\ &\textup{ and }\opnorm{J_{i' - 1 \ot i}(x)} \le 2 \tau_{i' - 1 \ot i}\} 
	\end{align*}
	Note that if $x \notin \mathcal{E}$, then $\exists i' < j'$ such that $Q_{j' \ot i'}(x) = 0$ and $Q_{j' \ot i'} \pfunc Q_{j \ot i}$ by definition of the order $\pfunc$. It follows that if $x \notin \mathcal{E}$, $\prod_{h \pfunc Q_{j \ot i}} h(x) = 0$, so $|\Delta_{j \ot i}(x, \nu)| = 0$. Otherwise, if $x \in \mathcal{E}$, by Claim~\ref{claim:jacobian_err} we have
	\begin{align*}
	\opnorm{J_{j \ot i}(x) - J_{j \ot i}(x + \nu)} \le \|\nu\|\sum_{i' = i}^{j} (2\tau_{j \ot i' + 1} + \bar{\tau} \|\nu \|)\bar{\tau}_{i'} (2\tau_{i' - 1 \ot 1} + \frac{\bar{\tau}}{2} \|\nu\|)2\tau_{i' - 1 \ot i}
	\end{align*}
	where we recall that $\tau_{i - 1 \ot i} = 1$. 
	Plugging this into~\eqref{eq:Delta_err-1} and using the fact that all functions $h \in \cfunc$ are bounded by 1 gives the desired statement.
	
	To prove ~\eqref{eq:Delta_err-j-m}, we simply apply the above argument with~\eqref{eq:jacobian_err-2}.
\end{proof}

\begin{claim}
	\label{claim:delta_err}
	In the setting of Lemma~\ref{lem:lip}, fix index $j$ with $0 \le j \le k$ and suppose that $J_{j \ot 1}$ is $\bar{\tau}$-Lipschitz. Then we can bound the error due to function composition by 
	\begin{align*}
	|\delta_j(x, \nu)| \le c_j (2 \tau_{j \to 1} + \frac{\bar{\tau}}{2} \|\nu\|)\|\nu\|
	\end{align*}
\end{claim}
\begin{proof}
	Starting from~\eqref{eq:delta_def}, we can first express $\delta_i(x, \nu)$ by 
	\begin{align*}
	\delta_j(x, \nu) = \gamma_{j} (x, \nu)Q_{j \ot 1}(x) \prod_{Q \succeq_{\cfunc} Q_{j \ot 1, 1}, Q \ne Q_{j \ot 1}} Q(x) \prod_{Q_{j \ot 1, 1} \pfunc Q} Q(x + \nu)
	\end{align*}
	as $Q_{j \ot 1} \pfunc Q_{j \ot 1, 1}$. First we note that by definition, $|\gamma_{j}(x, \nu)| \le c_j \|f_{j \ot 1}(x) - f_{j \ot 1}(x + \nu)\|$, as the function $g$ is $c_j$-Lipschitz in its $j$-th argument. Thus, since all functions $Q \in \cfunc$ are bounded by $1$, it follows that 
	\begin{align}
	|\delta_j(x, \nu)| &\le |\gamma_j(x, \nu)| Q_{j \ot 1}(x) \notag\\
	&\le c_j \|f_{j \ot 1}(x) - f_{j \ot 1}(x + \nu)\| \one[\le \tau_{j \ot 1}](\opnorm{J_{j \ot 1}(x)}) \notag\\
	&\le c_j (2 \tau_{j \to 1} + \frac{\bar{\tau}}{2} \|\nu\|)\|\nu\| \tag{by Claim~\ref{claim:finite_change}}
	\end{align}
\end{proof}

\begin{claim}
	\label{claim:all_lipschitz}
	In the setting of Lemma~\ref{lem:lip}, $\exists \bar{\tau}$ such that $\forall i \le j$, $J_{j \ot i}$ is $\bar{\tau}$-Lipschitz on a compact domain $\cD_0$. 
\end{claim}
\begin{proof}
	We first show inductively that $f_{i \ot 1}$ is Lipschitz for all $i$. The base case $f_{1 \ot 1}$ follows by definition, as $f_{1 \ot 1}$ is continuously differentiable and $\cD_0$ is a compact set. 
	
	Now we show the inductive step: first write $f_{i \ot 1} = f_{i} \circ f_{i - 1 \ot 1}$. By continuity, $\{f_{i - 1 \ot 1}(x) : x \in \cD_0\}$ is compact. Furthermore, $f_i$ is continuously differentiable under the assumptions of Lemma~\ref{lem:lip}. Thus, $f_i$ is Lipschitz on domain $\{f_{i - 1 \ot 1}(x) : x \in \cD_0\}$. As $f_{i \ot 1} = f_{i} \circ f_{i - 1 \ot 1}$ is the composition of Lipschitz functions by the inductive hypothesis, $f_{i \ot 1}$ is itself Lipschitz. 
	
	Now it follows that $\forall i$, $J_{i \ot i}$ is Lipschitz on $\cD_0$, as it is the composition of $Df_{i \ot i}$ and $f_{i - 1 \ot 1}$, both of which are Lipschitz. Finally, by the chain rule (Claim~\ref{claim:chain_rule}), we have that $J_{j \ot i} = J_{j \ot j} \cdots J_{i \ot i}$ is the product of Lipschitz functions, and therefore Lipschitz for all $i < j$. We simply take $\bar{\tau}$ to be the maximum Lipschitz constant of $J_{j \ot i}$ over all $i \le j$. 
\end{proof}

\begin{claim}
	\label{claim:diff_expansion}
	In the setting of Lemma~\ref{lem:lip}, 
	\begin{align*}
	|\tilde{z}(x) - \tilde{z}(x + \nu)| &\le \sum_{0 \le j \le k} |\delta_j(x, \nu)| + \sum_{1 \le i \le j \le k} |\Delta_{j \ot i}(x, \nu)| + \sum_{0 \le j \le k, m'} |\Delta_{j \ot 1, m'}(x, \nu)|
	\end{align*}
\end{claim}
\begin{proof}
	For $0 \le j \le k + 1$, define $z_j(x, \nu)$ by 
	\begin{align*}
	z_j(x, \nu) \triangleq g(f_{0 \ot 1}(x), \ldots, f_{j - 1 \ot 1}(x), f_{j \ot 1}(x + \nu), f_{j + 1 \ot 1}(x + \nu), \ldots, f_{k \ot 1}(x + \nu))
	\end{align*}
	Thus, $z_j(x, \nu)$ denotes $g \circ (f_{0 \ot 1} \otimes \ldots \otimes f_{k \ot 1})$ with the last $k + 1 - j$ inputs to $g$ depending on $x + \nu$ instead of $x$.  
Now we claim that by a telescoping argument (Claim~\ref{claim:telescope_product}), 
\begin{align}
\begin{split}
\tilde{z}(x) - \tilde{z}(x + \nu) = \\\sum_{0 \le j \le k} \delta_j(x, \nu) + \sum_{1 \le i \le j \le k}(z_k(j, \nu) - 1)\Delta_{j \ot i} + \sum_{0 \le j \le k, m'}(z_j(x, \nu) - 1)\Delta_{j \ot 1, m'}
\end{split} \label{eq:diff_expansion-1}
\end{align}
To see this, compute the sum in the order the following sequence of terms, which  corresponds to a traversal of $\cfunc$ in least-to-greatest order: 
{\small
\begin{align*}
\delta_k, (z_k(x, \nu) - 1)\Delta_{k \ot 1, 1}, \ldots, (z_k(x, \nu) - 1) \Delta_{k \ot 1, m'}, (z_k(x, \nu) - 1)\Delta_{k \ot 1}, \ldots, (z_k(x, \nu) - 1)\Delta_{k \ot k} \\
\vdots \\
\delta_1, (z_1(x, \nu) - 1)\Delta_{1 \ot 1, 1}, \ldots , (z_1(x, \nu) - 1)\Delta_{1 \ot 1, m'}, (z_1(x, \nu) - 1)\Delta_{1 \ot 1} \\
\delta_0, (z_0(x, \nu) - 1)\Delta_{0 \ot 1, 1}, \ldots, (z_0(x, \nu) - 1)\Delta_{0 \ot 1, m'}
\end{align*}
}
Now we simply apply triangle inequality on~\eqref{eq:diff_expansion-1} and use the fact that $z_j(x, \nu) - 1 \in [-1, 0] \ \forall 0 \le j \le k + 1$ to obtain the desired statement. 
\end{proof}
\begin{lemma}
	\label{lem:QJ_i_lipschitz}
	In the setting of Theorem~\ref{thm:lip}, fix $1 \le i \le \depth$ and define 
	\begin{align*}
	Q(J_i) \triangleq & \prod_{j \le i \le j'} \one[\le \kappa_{j' \ot j}] (\opnorm{D_{j'} \cdots D_{i + 1} J_i D_{i - 1} \cdots D_j})  \\
	&\times \prod_{j' \ge i + 1} \one[\le \kappa_{j' \ot i + 1}](\opnorm{D_{j'} \cdots D_{i + 1}})\times \prod_{j \le i - 1}\one[\le \kappa_{i - 1 \ot j}](\opnorm{D_{i - 1} \cdots D_{j}})
	\end{align*}
	Then $Q$ is $\laug_{J_i}$-Lipschitz in $J_i$, where $$\laug_{J_i} \triangleq \sum_{j \le i \le j'} \frac{4\kappa_{j' \ot i + 1} \kappa_{i - 1 \ot j}}{\kappa_{j' \ot j}}$$
	Here for convenience we use the convention that $\kappa_{i - 1 \ot i} = 1$.
\end{lemma}
\begin{proof}
	There are two cases: the condition $\opnorm{D_{j'} \cdots D_{i + 1}} \le 2\kappa_{j' \ot i + 1}$ and $\opnorm{D_{i - 1} \cdots D_{j}} \le 2\kappa_{i - 1 \ot j}$ for all $j' \ge i + 1$, $j \le i - 1$ either holds or does not hold. In the case that it does not hold, $Q$ is the constant function at $0$, and is certainly $\laug_{J_i}$-Lipschitz. In the case that the condition does hold, $\one[\le \kappa_{j' \ot j}] (\opnorm{D_{j'} \cdots D_{i + 1} J_i D_{i - 1} \cdots D_j})$ is $\frac{\kappa_{j' \ot i + 1} \kappa_{i - 1 \ot j}}{\kappa_{j' \ot j}}$-Lipschitz for all $j' \le i \le j$, and therefore their product is $\tilde{\kappa}_{J_i}$-Lipschitz. As the remaining indicators that do not depend on $J_i$ are constants in $[0, 1]$, it follows that $Q$ is $\tilde{\kappa}_{J_i}$-Lipschitz. 
\end{proof}

%% file: rnn.tex
\section{Application to Recurrent Neural Networks}\label{sec:recurrent}
In this section, we will apply our techniques to recurrent neural networks. Suppose that we are in a classification setting. For simplicity, we will assume that the hidden layer and input dimensions are $d$. We will define a recurrent neural network with $r- 1$ activation layers as follows using parameters $W, U, Y$, activation $\phi$ and input sequence $x = (x^{(0)}, \ldots, x^{(r-2)})$:
\begin{align*}
	F(x) &= Y h^{(2r-2)}(x)\\
	h^{(2i)}(x) &= \phi(h^{(2i-1)}(x) + u^{(i-1)}(x))\\
	h^{(2i- 1)}(x) &= Wh^{(2i-2)}(x)\\
	u^{(i-1)}(x) &= Ux^{(i-1)}
\end{align*}
where $h^{(0)}$  is set to be 0. Now following the convention of Section~\ref{sec:neural_net_main}, we will define the interlayer Jacobians. For odd indices $2i - 1$, $i \le r-1$, we simply set $Q_{2i - 1 \ot 2i- 1}$ to the constant function $x \mapsto W$. For even indices $2i$, $i \le r-1$, we set $Q_{2i \ot 2i}(x) \triangleq D\phi[h^{2i - 1}(x)+ u^{(i - 1)}(x)]$, the Jacobian of the activation applied to the input of $h^{(2i)}(x)$. Finally, we set $Q_{2r-1 \ot 2r -1}$ to be the constant function $x \mapsto Y$. Now for $i' > i$, we set $Q_{i' \ot i}(x) = Q_{i' \ot i'}(x) \cdots Q_{i \ot i}(x)$. If $i' < i$, we set $Q_{i' \ot i}$ to the identity matrix. 

With this notation in place, we can state our generalization bound for RNN's: 
\begin{theorem} \label{thm:rnngen}
	Assume that the activation $\phi$ is 1-Lipschitz with a $\bar{\sigma}_{\phi}$-Lipschitz derivative. With probability $1 - \delta$ over the random draws of $P_n$, all RNNs $F$ will satisfy the following generalization guarantee:
	{\scriptsize
		\begin{align*}
	\Exp_{(x, y) \sim P} [l_{\textup{0-1}}(F(x), y)] \le\\ \tilde{O}\left(\frac{\left((\kapparnh{r}a_Y t^{(r - 1)})^{2/3} + \sum_{i=1}^{r -1} {\kapparnh{i}}^{2/3}((a_{W} t^{(i - 1)})^{2/3} + (a_U t^{\textup{data}})^{2/3}) + \sum_{i = 1}^{r}(\kapparnj{i}b )^{2/3}\right)^{3/2}}{\sqrt{n}}\right)&\\ + \tilde{O}\left(r\sqrt{\frac{ \log(1/\delta)}{n}}\right)&
	\end{align*}}
	where $\kappaj{i} \triangleq \sum_{1 \le j \le 2i - 1 \le j' \le 2\depthnn - 1} \frac{\sigma_{j' \ot 2i} \sigma_{2i - 2 \ot j}}{\sigma_{j' \ot j}}$, and 
	{\scriptsize \begin{align*}
		\kappah{i} \triangleq \frac{1}{\poly(r)}+\frac{\sigma_{2r - 1 \ot 2i}}{\gamma} + \sum_{i \le i' < r} \frac{\sigma_{2i' \ot 2i}}{t^{(i')}} + \sum_{1 \le j \le j' \le 2\depthnn - 1} \sum_{\substack{j'' = \max\{2i, j\},\\ j'' \textup{ even }}}^{j'} \frac{\bar{\sigma}_\phi \sigma_{j'\ot j'' + 1} \sigma_{j'' - 1 \ot 2i} \sigma_{j'' - 1 \ot j}}{\sigma_{j' \ot j}}	\end{align*}}
	In these expressions, we define $\sigma_{j - 1 \ot j} = 1$, and:
	\[a_W \triangleq \poly(r)^{-1} + \|W^\top\|_{2, 1}, a_U \triangleq \poly(r)^{-1} + \|U^\top\|_{2, 1}\]\[ a_Y \triangleq \poly(r)^{-1} + \|Y^\top\|_{2, 1}, b \triangleq \poly(r)^{-1} +\|W\|_{1, 1}\]
	\[t^{(0)} =0, t^{\textup{data}} \triangleq \max_{x  \in P_n} \max_{i}\|x^{(i)}\|, \ t^{(i)} \triangleq \poly(r)^{-1}+\max_{x \in P_n} \|h^{(2i)}(x)\|\]
	\[\sigma_{j' \ot j} \triangleq \poly(r)^{-1}+\max_{x \in P_n} \opnorm{Q_{j' \ot j}(x)}, \textup{ and } \gamma \triangleq \min_{(x, y) \in P_n} [F(x)]_y - \max_{y' \ne y} [F(x)]_{y'} > 0\] 
	Note that the training error here is $0$ because of the existence of positive margin $\gamma$. 
\end{theorem}

Our proof follows the template of Theorem~\ref{thm:gen_union_bound}: we bound the Rademacher complexity of some augmented RNN loss. We then argue for generalization of the augmented loss and perform a union bound over all the choices of parameters. As the latter steps are identical to those in the proof of Theorem~\ref{thm:gen_union_bound}, we omit these and focus on bounding the Rademacher complexity of an augmented RNN loss. 

\begin{theorem}
	\label{thm:rnn_rad}
	Suppose that $\phi$ is $1$-Lipschitz with $\bar{\sigma}_{\phi}$-Lipschitz derivative. Define the following class of RNNs with bounded weight matrices:
	{\scriptsize
	\begin{align*}
	\cF \triangleq \left\{x \mapsto F(x) : \| W^\top\|_{2,1} \le a_Y, \|U^\top\|_{2,1}\le a_U, \|Y^\top\|_{2,1} \le a_Y, \|W\|_{1,1} \le b, \opnorm{W} \le \sigma_W, \opnorm{Y} \le \sigma_Y \right\}
	\end{align*}}
	and let $\sigma_{j' \ot j}$ be parameters that will bound the $j$ to $j'$ layerwise Jacobian for $j' \ge j$, where we set $\sigma_{2i \ot 2i} = 1$ and $\sigma_{2i - 1 \ot 2i - 1} = \sigma_W$ for $i \le r -1$, $\sigma_{2r - 1 \ot 2r - 1} = \sigma_Y$. Let $t^{(i)}$ be parameters bounding the layer norm after applying the $i$-th activation, and let $t^{(0)} = 0, t^{\textup{data}} =\max_{x \in P_n} \max_i \|x^{(i)}\|$. Define the class of augmented losses 
	\begin{align*}
	\cL_{\textup{rnn-aug}} \triangleq \left\{(l_{\gamma} - 1)\circ F \prod_{i = 1}^{\depthnn - 1} \one[\le t^{(i)}](\|h^{(2i)}\|) \prod_{1 \le j < j' \le 2\depthnn - 1} \one[\le \sigma_{j' \ot j}] (\opnorm{Q_{j' \ot j}}) + 1: F \in \cF\right\}
	\end{align*}
	and define for $1 \le i \le \depthnn$, $\kappaj{i}, \kappah{i}$ meant to bound the influence of the matrix $W^{(i)}$ on the Jacobians and hidden variables, respectively as in~\eqref{eq:kappa_j},~\eqref{eq:kappa_h}.
	Then we can bound the empirical Rademacher complexity of the augmented loss class by
	{\scriptsize
	\begin{align*}
	\rad(\cL_{\textup{rnn-aug}}) =\\ \tilde{O}\left(\frac{\left((\kapparnh{r}a_Y t^{(r - 1)})^{2/3} + \sum_{i=1}^{r -1} {\kapparnh{i}}^{2/3}((a_{W} t^{(i - 1)})^{2/3} + (a_U t^{\textup{data}})^{2/3}) + \sum_{i = 1}^{r}(\kapparnj{i}b )^{2/3}\right)^{3/2}}{\sqrt{n}}\right)
	\end{align*}}
	where $\kapparnh{i}, \kapparnj{i}$ are defined in Theorem~\ref{thm:rnngen}.
\end{theorem}
\begin{proof}
	We will associate the family of losses $\cL_{\textup{rnn-aug}}$ with a computational graph structure on internal nodes $H_1, H_2, \ldots, H_{2r - 1}$, $J_1, \ldots, J_{2r - 1}$, $K_0,\ldots,K_{r-2}$, input nodes $H_0, I_0, \ldots, I_{r- 2}$, and output node $O$ with the following edges: 
	\begin{enumerate}
		\item Nodes $H_i, J_i$ will point towards the output $O$.
		\item Node $H_i$ will point towards nodes $H_{i + 1}$ and $J_{i + 1}$. 
		\item Node $K_{i-1}$ will point towards node $H_{2i}$ and node $J_{2i}$. 
		\item Node $I_{i}$ will point towards node $K_i$. 
	\end{enumerate}
	We now define the composition rules at each node: 
	\begin{align*}
		\fR_{H_{2i}}&= \{(h, k) \mapsto \phi(h + k)\}\\
		\fR_{H_{2i - 1}} &= \{h\mapsto Wh : \|W^\top \|_{2, 1} \le a_W, \opnorm{W} \le \sigma\} \textup{ for } 2\le i \le r - 1\\
		\fR_{H_{2r - 1}} &= \{h \mapsto Yh : \|Y^\top\|_{2, 1} \le a_Y, \opnorm{Y} \le \sigma_Y\}\\
		\fR_{J_{2i}} &= \{(h, k) \mapsto D\phi[h + k]\}\\
		\fR_{K_i} &= \{x \mapsto Ux : \|U^\top \|_{2, 1} \le a_U\}
	\end{align*}
	Finally, nodes $J_{2i-1}$ will have composition rule $R_{J_{2i-1}} = DR_{H_{2i-1}}$. Finally, the output node $O$ will have composition rule 
	\begin{align*}
		R_O(x, h_1, \ldots, h_{2r -1}, D_1, \ldots, D_{2r - 1}) \triangleq \\(l_{\gamma}(h_{2r - 1}) - 1) \prod_{i = 1}^{r- 1} \one[\le t^{(i)}] (\|h_{2i}\|) \prod_{1 \le j < j' \le 2r - 1} \one[\le \sigma_{j' \ot j}](\opnorm{D_{j'} \cdots D_j}) + 1
	\end{align*}
	Note that the family of functions computed by this computation graph family is a strict superset of $\cL_{\textup{rnn-aug}}$ (as we technically allow $R_{H_{2i - 1}}$, $R_{H_{2i' -1}}$ to use different matrices $W$). We will refer to this resulting family as $\tilde{\cG}$.
	
	First, we claim that $\tilde{G}$ satisfies the release-Lipschitz condition, with Lipschitz constants $\kapparnh{i}$ for nodes $H_{2i -1}$ and $K_{i - 1}$, and $\kapparnj{i}$ for nodes $J_{2i - 1}$. (As we will see later, the Lipschitzness of nodes $V_{2i}$, $J_{2i}$ will not matter because the composition rules are function classes with log covering number 0.)
	
	To see this, we note that if we release $K_0, \ldots, K_{r - 2}$ from the graph and set them to fixed values, the resulting induced graph family is simply the Lipschitz augmentation of Section~\ref{sec:lipschitzaug} for the sequential graph family on nodes $H_{0}, \ldots, H_{2r-1}$ and an un-augmented output. Thus, the machinery of Theorem~\ref{thm:lip}~applies here, and we can conclude that this reduced graph family is $\kapparnh{i}$-release-Lispchitz for nodes $H_{2i -1}$ and $\kapparnj{i}$-release-Lipschitz for nodes $J_{2i - 1}$. Since this holds for any choice of $K_0, \ldots, K_{r - 2}$, we can draw the same conclusion about $\tilde{\cG}$, the augmented family that is not reduced. However, by nature of the composition rules in $\tilde{\cG}$, the Lipschitzness of $H_{2i -1}$ and $K_{i - 1}$ must be identical (as $f(x + y)$ must have the same worst-case Lipschitz constant in $x$ and $y$ for any function $f$). Thus, we get that $\tilde{G}$ satisfies release-Lipschitzness with constants $\kapparnh{i}$ for nodes $H_{2i - 1}$, $K_{i - 1}$, and $\kapparnj{i}$ for nodes $J_{2i - 1}$. 
	
	With this condition established, we can complete the proof via the same covering number argument as in Theorem~\ref{thm:nnrad}.
\end{proof}

Now as in the proof of Theorem~\ref{thm:gen_union_bound}, we first observe that the augmented loss upper bounds the 0-1 classification loss, giving us a 0-1 test error bound. We then apply the same union bound technique over parameters $\gamma, t^{(i)}, \sigma_{j' \ot j}, a_W, a_U, a_Y$, as in the proof of Theorem~\ref{thm:gen_union_bound}.

%% file: relu.tex
\section{ReLU Networks}\label{app:relu}

In this section, we apply our augmentation technique to relu networks to produce a generalization bound similar to that of \cite{nagarajan2018deterministic}, which is polynomial in the Jacobian norms, hidden layer norms, and inverse pre-activations. 

Recall the definition of neural nets in Example~\ref{ex:neuralnet}: the neural net with parameters $\{W^{(i)}\}$ and activation $\phi$ is defined by 
\begin{align*}
	F(x) = W^{(r)} \phi(\cdots \phi(W^{(1)}x) \cdots)
\end{align*}
For this section, we will set $\phi$ to be the relu activation. We also use the same notation for layers and indexing as Section~\ref{sec:neural_net_main}. We first state our generalization bound for relu networks: 

\begin{theorem}\label{thm:relunet}
	Fix reference matrices $\{A^{(i)}\}, \{B^{(i)}\}$. With probability $1 - \delta$ over the random draws of the data $P_n$, all neural networks $F$ with relu activations parameterized by $\{W^{(i)}\}$ will have the following generalization guarantee
	{\scriptsize
			\begin{align*}
			\Exp_{(x, y) \sim P} [l_{\textup{0-1}}(F(x), y)] \le \tilde{O}\left(\frac{\left(\sum_{i} (\kapparh{i}a^{(i)} t^{(i - 1)})^{2/3} + (\kapparj{i}b^{(i)})^{2/3}\right)^{3/2}}{\sqrt{n}} + r\sqrt{\frac{ \log(1/\delta)}{n}}\right)
			\end{align*}}
		 where 
 \begin{align}
 \begin{split}
	\kapparj{i} &\triangleq \sum_{1 \le j \le 2i - 1 \le j' \le 2\depthnn - 1} \frac{\sigma_{j' \ot 2i} \sigma_{2i - 2 \ot j}}{\sigma_{j' \ot j}}\\
\kapparh{i} &\triangleq \frac{1}{\poly(r)}+\frac{\sigma_{2r - 1 \ot 2i}}{\gamma} + \sum_{i \le i' < r} \frac{\sigma_{2i' \ot 2i}}{t^{(i')}} + \frac{\sigma_{2i' - 1\ot 2i}}{\gamma^{(i')}} \label{eq:kapparelu}
\end{split}
\end{align}
In these expressions, we define $\sigma_{j - 1 \ot j} = 1$, $\gamma^{(i)}$ to be the minimum pre-activation after the $i$-th weight matrix over all coordinates in the $i$-th layer and all datapoints: 
\begin{align*}
	\gamma^{(i)} \triangleq \min_{x \in P_n} \min_j |[F_{2i - 1\ot 1}(x)]_j|
\end{align*} where $[F_{2i - 1 \ot 1}(x)]_j$ indexes the $j$-th coordinate of $F_{2i - 1 \ot 1}(x)$, and additionally use
\[a^{(i)} \triangleq \poly(r)^{-1} + \|{W^{(i)}}^\top - {A^{(i)}}^\top\|_{2, 1}, b^{(i)} \triangleq \poly(r)^{-1} + \|{W^{(i)}} - {B^{(i)}}\|_{1, 1}\]
\[t^{(0)} \triangleq \poly(r)^{-1}+\max_{x  \in P_n} \|x\|, \ t^{(i)} \triangleq \poly(r)^{-1}+\max_{x \in P_n} \|F_{2i \ot 1}(x)\|\]
\[\sigma_{j' \ot j} \triangleq \poly(r)^{-1}+\max_{x \in P_n} \opnorm{Q_{j' \ot j}(x)}, \textup{ and } \gamma \triangleq \min_{(x, y) \in P_n} [F(x)]_y - \max_{y' \ne y} [F(x)]_{y'} > 0\] 
Note that we assume the existence of a positive margin, so the training error here is $0$. 
	
\end{theorem}
We note that compared to Theorem~\ref{thm:gen_union_bound}, $\kapparj{i} = \kappaj{i}$, but $\kapparh{i}$ now has a dependence on the preactivations $\gamma^{(i)}$, as in~\citet{nagarajan2018deterministic}.

We provide a proof sketch of Theorem~\ref{thm:relunet} here. We first bound the Rademacher complexity some family of augmented losses, specified precisely in Theorem~\ref{thm:relunnrad}. The rest of the argument then follows the same way as the proof of Theorem~\ref{thm:gen_union_bound}: using Rademacher complexity to argue that the augmented losses generalize, applying the fact that the augmented losses upper-bound the 0-1 loss, and then union bounding over all choices of parameters.

\begin{theorem} \label{thm:relunnrad}
	Following the definitions in Theorem~\ref{thm:nnrad}, let $\mathcal{F}$ denote the class of neural networks, $\sigma_{j' \ot j}$ be parameters intended to bound the spectral norm of the $j$ to $j'$ layerwise Jacobian, and $t^{(i)}$ be parameters bounding the layer norm after applying the $i$-th activation. Define $\gamma^{(i)}$ as parameters intended to lower bound the minimum preactivations after the $i$-th linear layer. Define the class of augmented losses
{\scriptsize
	\begin{align*}
	\cL_{\textup{relu-aug}} \triangleq \left\{(l_{\gamma} - 1)\circ F \prod_{i = 1}^{\depthnn - 1} \one[\le t^{(i)}](\|F_{2i \ot 1}\|) \one[\ge \gamma^{(i)}](\min_j |[F_{2i - 1 \ot 1}]_j|)\prod_{1 \le j < j' \le 2\depthnn - 1} \one[\le \sigma_{j' \ot j}] (\opnorm{Q_{j' \ot j}}) + 1: F \in \cF\right\}
	\end{align*}}
	where $\one[\ge \gamma^{(i)}] \triangleq 1 - \one[\le \gamma^{(i)}/2]$. Define for $1 \le i \le r$, $\kapparj{i}, \kapparh{i}$ meant to bound the influence of the matrix $W^{(i)}$ on the Jacobians and hidden variables, respectively, as in~\eqref{eq:kapparelu}. Then the augmented loss class $\cL_{\textup{relu-aug}}$ has empirical Rademacher complexity upper bound 
	\begin{align*}
		\rad(\cL_{\textup{relu-aug}}) = \tilde{O}\left(\frac{\left(\sum_i (\kapparh{i} a^{(i)} t^{(i- 1)})^{2/3} + (\kapparj{i}b^{(i)})^{2/3}\right)^{3/2}}{\sqrt{n}}\right) 
	\end{align*}
\end{theorem}
Note the differences with Theorem~\ref{thm:nnrad}: the augmented loss class $\cL_{\textup{relu-aug}}$ now includes the additional indicators $\one[\ge \gamma^{(i)}](\min_j |[F_{2i - 1 \ot 1}]_j|)$, and we use the Lipschitz constants $\kapparh{i}, \kapparj{i}$ defined in Theorem~\ref{thm:relunet}.
\begin{proof}[Proof sketch]
	As in the proof of Theorem~\ref{thm:nnrad}, associate the loss class $\cL_{\textup{relu-aug}}$ with a family $\augG$ of computation graphs on internal nodes $V_1, \ldots, V_{2r - 1}, J_1, \ldots, J_{2r - 1}$ as follows: define the graph structure to be identical to the Lipschitz augmentation of a sequential computation graph family (Figure~\figaugmentation) and define the composition rules
	\begin{align*}
	\fR_{V_{2i}} &= \{\phi\}\\
	\fR_{V_{2i - 1}} &= \{h \mapsto Wh : \|W^\top - {A^{(i)}}^\top\|_{2, 1} \le a^{(i)}, \|W - {B^{(i)}}\|_{1, 1} \le b^{(i)}, \opnorm{W} \le \sigma^{(i)}\}
	\end{align*}
	Assign to the $J_i$ nodes composition rule $R_{J_i} = DR_{V_i}$, and finally, assign to the output node $O$ the composition rule
	\begin{align*}
	R_O(x, v_1, \ldots, v_{2r - 1}, D_1, \ldots, D_{2r - 1}) \triangleq \\(l_\gamma(v_{2r - 1}) - 1) \prod_{i = 1}^{r - 1}\one[\le t^{(i)}] (\|v_{2i}\|) \one[\ge \gamma^{(i)}] (\min_j |[v_{2i - 1}]_j|) \prod_{1 \le j \le j' \le 2r - 1} \one[\le \sigma_{j' \ot j}](\opnorm{D_{j'} \cdots D_j}) + 1
	\end{align*}
	The resulting family of computation graphs will compute $\cL_{\textup{relu-aug}}$. Now we claim that $\augG$ is $\kapparh{i}$-release-Lipschitz in nodes $V_{2i - 1}$ and $\kapparj{i}$-release-Lipschitz in nodes $J_{2i - 1}$. (Note that the Lipschitzness of nodes $V_{2i}, J_{2i}$ will not matter because the associated function classes and singletons and therefore have a log covering number of 0 anyways). 
	
	The argument for the $\kapparj{i}$-release-Lipschitzness of $J_{2i - 1}$ follows analogously to the argument of Lemma~\ref{lem:QJ_i_lipschitz} and Theorem~\ref{thm:nnrad}. 
	
	To see the $\kapparh{i}$-release-Lipschitzness of $V_{2i - 1}$, we first note that we can account for the instantaneous change in the graph output given a change to $V_{2i - 1}$ as a sum of the following: 1) the change in $l_\gamma(V_{2r - 1}) - 1$ multiplied by the other indicators, 2) the change in the term $\one[\le t^{(i')}](\|V_{2i}\|) \one[\ge \gamma^{(i)}](\min_j |[V_{2i - 1}]_j|)$ multiplied by the other indicators, and 3) the change in $\one[\le \sigma_{j' \ot j}](\opnorm{J_{j'} \cdots J_j})$ multiplied by the other indicators. The term 1) can be computed as $\frac{\sigma_{2r - 1 \ot 2i}}{\gamma}$, term 2) can be accounted for by $\frac{\sigma_{2i' \ot 2i}}{t^{(i')}} + \frac{\sigma_{2i' - 1 \ot 2i}}{\gamma^{(i')}}$, and finally the term 3) is 0 because as relu is piecewise-linear, the instantaneous change in the Jacobian is 0 if all preactivations are bounded away from 0, and in the case that the preactivations are not bounded away from 0, the indicator $\one[\ge \gamma^{(i)}](\min_j |[V_{2i - 1}]_j|)$ takes value 0. The same steps as Lemma~\ref{lem:lip} can be used to formalize this argument.
	
	Finally, to conclude the desired Rademacher complexity bounds given the release-Lipschitzness, we apply the same reasoning as in Theorem~\ref{thm:nnrad}.
	\end{proof}

%% file: experiment_details.tex
\section{Additional Experimental Details}
\subsection{Implementation Details for Jacobian Regularizer}
\label{sec:exp_details}
For all settings, we train for 200 epochs with learning rate decay by a factor of 0.2 at epochs 60, 120, and 150. We additionally tuned the value of $\lambda$ from values $\{0.1, 0.05, 0.01\}$ for each setting: for the experiments displayed in Figure~\ref{fig:jreg}, we used the following values: 
\begin{enumerate}
	\item Low learning rate: $\lambda = 0.1$
	\item No data augmentation: $\lambda = 0.1$
	\item No BatchNorm: $\lambda = 0.05$
\end{enumerate}

For all other hyperparameters, we use the defaults in the PyTorch WideResNet implementation: \url{https://github.com/xternalz/WideResNet-pytorch}, and we base our code off of this implementation. We report results from a single run as the improvement with Jacobian regularization is statistically significant. We train on a single NVIDIA TitanXp GPU.

\subsection{Empirical Scaling of our Complexity Measure with Depth}
\label{sec:empirical_comparison}
In this section, we empirically demonstrate that the leading term of our bounds can exhibit better scaling in depth than prior work. 

\begin{figure}[h]
	\centering
	\includegraphics[width=0.4\textwidth]{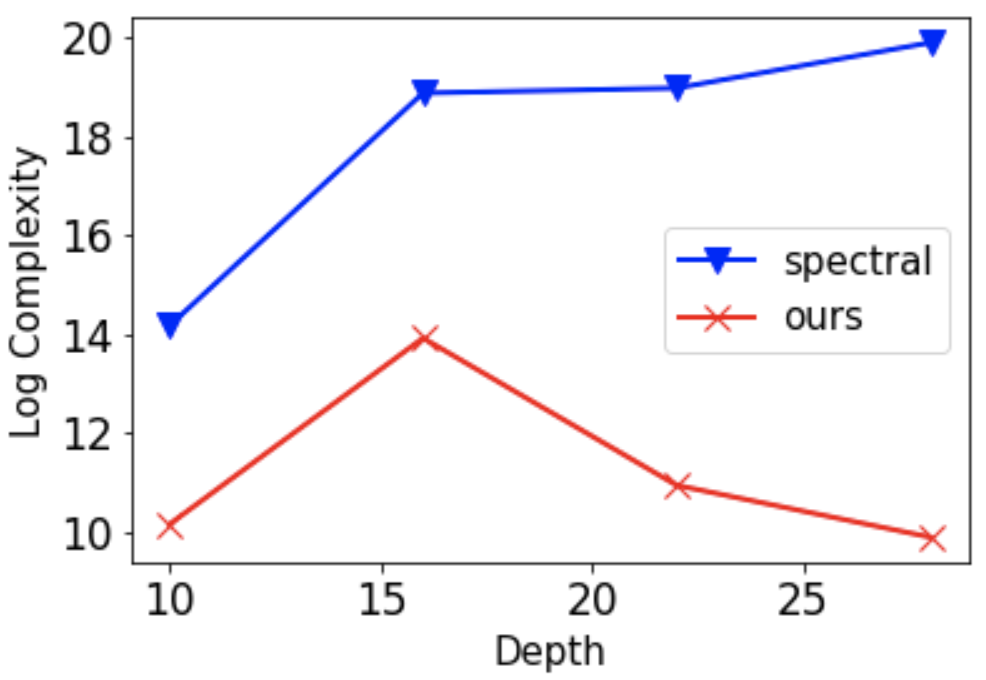}	
	\caption{Log leading terms for spectral vs. our bound on WideResNet trained on CIFAR10 using different depths.} \label{fig:bounds}
\end{figure}

We compute leading terms of our bound: $\frac{\sum_{i} \max_{x \in P_n} \|h^{(i)}(x)\|_2 \max_{x \in P_n} \opnorm{J^{(i)}(x)\|}}{\gamma}$, where $i$ ranges over the layers, $h^{(i)}$, $J^{(i)}$ denote the $i$-th hidden layer and Jacobian of the output with respect to the $i$-th hidden layer, respectively, and $\gamma$ denotes the smallest positive margin on the training dataset. We compare this quantity with that of the bound of~\citep{bartlett2017spectrally}: $\prod_i \opnorm{W^{(i)}}/\gamma$. In Figure~\ref{fig:bounds}, we plot this comparison for WideResNet\footnote{Our bound as stated in the paper technically does not apply to ResNet because the skip connections complicate the Lipschitz augmentation step. This can be remedied with a slight modification to our augmentation step, which we omit for simplicity.} models of depths 10, 16, 22, 28 trained on CIFAR10. For all models, we remove data augmentation to ensure that our models fit the training data perfectly. We train each model for 50 epochs, which is sufficient for perfectly fitting the training data, and start from an initial learning rate of 0.1 which we decrease by a factor of 10 at epoch 30. All other parameters are set to the same as their defaults in the PyTorch WideResNet implementation: \url{https://github.com/xternalz/WideResNet-pytorch}. We plot the final complexity measures computed on a single model. We note that our models are trained with Batchnorm. At test time, these Batchnorm layers compute affine transformations, so we compute the bound by merging these transformations with the adjacent linear layer. 

Figure~\ref{fig:bounds} demonstrates that our complexity measure can be much lower than the spectral complexity. Furthermore, in Figure~\ref{fig:bounds}, our complexity measure appears to scale well with depth for WideResNet models. 

%% file: helper_lemmas.tex
\section{Toolbox}

\begin{claim}
	Consider the function $u : [0, 1] \times [0, 1] \mapsto \R$ defined as follows: $u(x_1, x_2) = (x_1 - 1)x_2 + 1$. Then the following statements hold:
	\begin{enumerate}
		\item The function $u$ outputs values in $[0, 1]$. 
		\item $u(x_1, x_2) \ge x_1$. 
		\item $u(u(x_1, x_2), x_3) = u(x_1, x_2x_3)$. 
	\end{enumerate}
	\label{claim:indicator_stack}
\end{claim}
\begin{proof}
	First, we note that $u(x_1, x_2) = x_1 x_2 + 1 - x_2 \le x_2 + 1 - x_2 = 1$. Furthermore, $u(x_1, x_2) \ge x_1 x_2 + x_1(1 - x_2) = x_1$, which completes the proof of statements 1 and 2. To prove the third statement, we note that $u(u(x_1, x_2), x_3) = (x_1 x_2 + 1 - x_2)x_3 + 1 - x_3 = x_1 x_2 x_3 + 1 - x_2 x_3 = u(x_1, x_2x_3)$. 
\end{proof}

\begin{claim}[Chain rule{~\cite{wiki:chain_rule}}]\label{claim:chain_rule}
	The Jacobian of a composition of a sequence of  functions $f_1,\dots, f_k$ satisfies 
	\begin{align}
	D f_{k\ot 1}(x) = Df_k(f_{(k-1)\ot 1}(x)) \cdot Df_{k-1}(f_{(k-2)\ot 1}(x)) \cdots Df_{2}(f_1(x)) \cdot Df_1(x)
	\end{align}
	where the $\cdot$ notations are standard matrix multiplication. For simplicity, we also write in the function form:
	\begin{align}
	D f_{k\ot 1} = (Df_k \circ f_{(k-1)\ot 1})\cdot (Df_{k-1} \circ f_{(k-2)\ot 1}) \cdots (Df_{2}\circ f_1) \cdot Df_1
	\end{align}
\end{claim}

\begin{claim}[Telescoping sum]\label{claim:telescope_product}
	Let $p_1,\dots, p_m$ and $q_1\dots q_m$ be two sequence of  functions from $\R^d$ to $\R$. Then, 
	{\small
	\begin{align}
	p_1p_2\cdot p_m - 			q_1q_2\cdot q_m = (p_1-q_1)p_2\cdots p_m + q_1(p_2-q_2)p_3\cdots p_m + \cdots+ q_1\cdots q_{m-1}(p_m-q_m)
	\end{align}}
\end{claim}

\begin{claim}[Bounding function differences]
	\label{claim:finite_change}
	Let $f : \cD \rightarrow \cD'$, and consider the total derivative $Df$ operator mapping $\cD$ to a linear operator between normed spaces $\cD$ to $\cD'$. Suppose that $Df[x]$ is $\kappa$-Lipschitz in $x$, in the sense that $\opnorm{Df[x] - Df[x + \nu]} \le \kappa \|\nu\|$, where $\opnorm{\cdot}$ is the operator norm induced by $\cD$ and $\cD'$. Then 
	\begin{align}
	\|f(x) - f(x + \nu)\| \le (\opnorm{Df[x]} + \frac{\kappa}{2} \|\nu \|)\|\nu \| \label{eq:finite_change-1}
	\end{align}
	Furthermore, 
	\begin{align}
	\|f(x) - f(x + \nu)\| \one[\le \tau_f](\opnorm{Df[x]}) \le (2\tau_f + \frac{\bar{\tau}}{2}\|\nu\|)\|\nu\| \label{eq:finite_change-2}
	\end{align}
\end{claim}
\begin{proof}
	We write $f(x + \nu) - f(x) = \left(\int_{t = 0}^1 Df[x + t\nu]dt \right) \nu$. Now we note that 
	\begin{align}
	\opnorm{\int_{t = 0}^1Df[x + t\nu]dt} &\le \int_{t = 0}^1 \opnorm{Df[x + t\nu]}dt \tag{by triangle inequality}\\ 
	&\le \int_{t = 0}^1 (\opnorm{Df[x]} + t \kappa \|\nu \|)dt \tag {by Lipschitzness of $Df$}\\
	&\le \opnorm{Df[x]} + \frac{\kappa}{2} \|\nu\| \label{eq:intnormbound}
	\end{align}
	Thus, 
	\begin{align}
	\|f(x + \nu) - f(x)\| &\le \opnorm{\int_{t = 0}^1Df[x + t\nu]dt} \|\nu\| \nonumber \\
	&\le \left(\opnorm{Df[x]} + \frac{\kappa}{2} \|\nu\|\right) \|\nu\| \tag{by~\eqref{eq:intnormbound}}
	\end{align}
	which proves~\eqref{eq:finite_change-1}. 
	
	To prove~\eqref{eq:finite_change-2}, we consider two cases.first, if $\opnorm{Df[x]} > 2\tau_f$, then $\one[\le \tau_f](\opnorm{Df[x]}) = 0$ so~\eqref{eq:finite_change-2} immediately holds. Otherwise, if $\opnorm{Df[x]} \le 2\tau_f$, we can plug this into \eqref{eq:finite_change-1} to obtain~\eqref{eq:finite_change-2}, as desired.
\end{proof}